\documentclass[12pt,a4paper]{article}

\usepackage{setspace}
\usepackage{graphicx}
\usepackage[numbers]{natbib}

\usepackage{amsmath}
\usepackage{amsfonts}
\usepackage{amssymb}
\usepackage{float}
\usepackage{hyperref}
\usepackage{color}

\newcounter{thmcount}
\newtheorem{conjecture}[thmcount]{Conjecture}
\newtheorem{corollary}[thmcount]{Corollary}
\newtheorem{definition}[thmcount]{Definition}
\newtheorem{lemma}[thmcount]{Lemma}
\newtheorem{proposition}[thmcount]{Proposition}
\newtheorem{theorem}[thmcount]{Theorem}

\newenvironment{proof}{\begin{trivlist} \item[]
{\bf Proof.}}{\nolinebreak
\hfill \rule{2mm}{2mm} \end{trivlist}}

\numberwithin{equation}{section}

\begin{document}
\begin{titlepage}
	\centering
	{\Large\scshape On the Identification and Mitigation of Weaknesses in the
	Knowledge Gradient Policy for Multi-Armed Bandits \par}
	\vspace{1cm}
 	{\scshape James Edwards}\\
 	{\itshape STOR-i Centre for Doctoral Training, Lancaster University,
 	Lancaster,LA1 4YF, UK}
 	\texttt{j.edwards4@lancaster.ac.uk}\\
 	\vspace{0.2cm}
 	{\scshape Paul Fearnhead}\\
 	{\itshape Department of Mathematics and Statistics, Lancaster University,
 	Lancaster, LA1 4YF, UK}
 	\texttt{p.fearnhead@lancaster.ac.uk}\\
	\vspace{0.2cm}
	{\scshape Kevin Glazebrook}\\
	{\itshape Department of Management Science, Lancaster University, LA1 4YX,
	 UK}
	\texttt{k.glazebrook@lancaster.ac.uk}\\

\begin{abstract}
The Knowledge Gradient (KG) policy was originally proposed for offline ranking
and selection problems but has recently been adapted for use in online
decision-making in general and multi-armed bandit problems (MABs) in particular.
We study its use in a class of exponential family MABs and identify weaknesses,
including a propensity to take actions which are dominated with respect to both
exploitation and exploration. We propose variants of KG which avoid such errors.
These new policies include an index heuristic which deploys a KG approach to
develop an approximation to the Gittins index. A numerical study shows this
policy to perform well over a range of MABs including those for which index
policies are not optimal. While KG does not take dominated actions when bandits
are Gaussian, it fails to be index consistent and appears not to enjoy a
performance advantage over competitor policies when arms are correlated to
compensate for its greater computational demands.
\end{abstract}
\vfill
\end{titlepage}

\section{Introduction}
\label{sec:intro}
Bayes sequential decision problems (BSDPs) constitute a large class of
optimisation problems in which decisions (i) are made in time sequence and
(ii) impact the system of interest in ways which may be not known or only
partially known. Moreover, it is possible to learn about unknown system
features by taking actions and observing outcomes. This learning is modelled
using a Bayesian framework. One important subdivision of BSDPs is between
offline and online problems. In offline problems some decision is required
at the end of a time horizon and the purpose of actions through the horizon
is to accumulate information to support effective decision-making at its
end. In online problems each action can bring an immediate payoff in
addition to yielding information which may be useful for subsequent
decisions. This paper is concerned with a particular class of online
problems although it should be noted that some of the solution methods have
their origins in offline contexts.

The sequential nature of the problems coupled with imperfect system
knowledge means that decisions cannot be evaluated alone. Effective
decision-making needs to account for possible future actions and associated
outcomes. While standard solution methods such as stochastic dynamic
programming can in principle be used, in practice they are
computationally impractical and heuristic approaches are generally required.
One such approach is the knowledge gradient (KG) heuristic.
\citet{gupta1996bayesian} originated KG for application to offline ranking and
selection problems. After a period of time in which it appears to have been
studied little, \citet{frazier2008knowledge} expanded on KG's
theoretical properties. It was adapted for use in online decision-making by
\citet{ryzhov2012knowledge} who tested it on multi-armed bandits (MABs) with
Gaussian rewards. They found that it performed well against an index policy
which utilised an analytical approximation to the Gittins index; see
\citet{gittins2011multi}. \citet{ryzhov2010robustness} have investigated the
use of KG to solve MABs with exponentially distributed rewards while
\citet{powell2012optimal} give versions for Bernoulli, Poisson and uniform
rewards, though without testing performance. They propose the method as an
approach to online learning problems quite generally, with particular emphasis
on its ability to handle correlated arms. Initial empirical results were
promising but only encompassed a limited range of models. This paper utilises an
important sub-class of MABs to explore properties of the KG heuristic for
online use. Our investigation reveals weaknesses in the KG approach. We
\textit{inter alia }propose modifications to mitigate these weaknesses.

In Section \ref{sec:expMAB} we describe a class of exponential family MABs
that we will focus on, together with the KG policy for them. Our main analytical
results revealing weaknesses in KG are given in Section \ref{sec:dom}. Methods
aimed at correcting these KG errors are discussed in Section \ref{sec:newpols}
and are evaluated in a computational study which is reported in Section
\ref{sec:comp1}. In this study a range of proposals are assessed for Bernoulli
and Exponential versions of our MAB models. Gaussian MABs have characteristics
which give the operation of KG distinctive features. The issues for such models are
discussed in Section \ref{sec:gauss}, together with an associated computational
study in Section \ref{sec:comp2}. Section \ref{sec:conclusion} identifies some
key conclusions to be drawn.

\section{A class of exponential family multi-armed bandits}
\label{sec:expMAB}
\subsection{Multi-Armed Bandit Problems for Exponential Families}
\label{sec:MAB}
We consider multi-armed bandits (MABs) with geometric discounting operating
over a time horizon $T\in\mathbb{Z}^{+}\cup\{\infty\} $ which may
be finite or not. Rewards are drawn from exponential families with independent
conjugate priors for the unknown parameters. More specifically the set up is as follows:

\begin{enumerate}
\item At each decision time $t\in\{ 0,1,\ldots ,T-1\}$ an
action $a\in\{1,\ldots,k\} $ is taken. Associated with each
action, $a$, is an (unknown) parameter, which we denote as $\theta_a$. Action
$a$ (\textit{pulling arm }$a$) yields a reward which is drawn from the density
(relative to some $\sigma$-finite measure on $\mathbb{R}$)
\begin{equation}
f(y\mid \theta _{a}) =e^{\theta_{a}y-\psi ( \theta
_{a})}, y\in \Omega,\theta _{a}\in \Theta, 
\end{equation}%
where $\Omega \subseteq\mathbb{R}$ is the support of $f,$ $\psi$ is a cumulant
generating function and parameter space $\Theta \subseteq\mathbb{R}$ is such
that $\psi\left(\theta \right)<\infty,\;\forall\theta \in \Theta$. Reward
distributions are either discrete or absolutely continuous, with $%
\Omega $ a discrete or continuous interval $\left[ \min \Omega ,\max \Omega %
\right] $ where $-\infty \leq \min \Omega <\max \Omega \leq \infty $. We shall take a
Bayesian approach to learning about the parameters $\theta_a$.

\item We assume independent conjugate priors for the unknown $\theta_{a}$
with Lebesgue densities given by 
\begin{equation}
g( \theta _{a}\mid \Sigma _{a},n_{a}) \propto e^{\Sigma
_{a}\theta _{a}-n_{a}\psi(\theta _{a})},\theta _{a}\in \Theta
,1\leq a\leq k,
\end{equation}%
where $\Sigma_a$ and $n_a$ are known hyper-parameters. This then defines a
predictive density
\begin{equation}
p(y\mid \Sigma _{a},n_{a})=\int_{\Theta }f( y\mid \theta _{a}) g(
\theta _{a}\mid \Sigma _{a},n_{a}) d\theta _{a} 
\end{equation}%
which has mean $\frac{\Sigma _{a}}{n_{a}}$. Bayesian updating following an
observed reward $y$ on arm $a$ produces a posterior
$p(\theta_a|y)=g(\theta_a|\Sigma_a+y,n_a+1)$. Thus at each time we can define an
arm's \emph{informational state} as the current value of hyper-parameters
$\Sigma_a,n_a$, such that the posterior for $\theta_a$ given the observations to
date is $g(\theta_a|\Sigma_a,n_a)$. The posterior for each arm is
independent so the informational states of arms not pulled at $t$ are unchanged.

\item The total return when reward $y_{t}$ is received at time $t$ is given
by $\sum\nolimits_{t=0}^{T-1}\gamma ^{t}y_{t},$ where discount rate $\gamma 
$ satisfies either $0<\gamma \leq 1$ when $T<\infty $ or $0<\gamma <1$ when $%
T=\infty $. The objective is to design a policy, a rule for choosing
actions, to maximise the \textit{Bayes' return, }namely the total return
averaged over both realisations of the system and prior information.
\end{enumerate}

The current informational state for all arms, denoted $(\mathbf{\Sigma
},\mathbf{n)}=\left\{ (\Sigma _{a},n_{a}),1\leq a\leq k\right\} $ summarises all the information in the
observations up to the current time. 

When $0<\gamma <1,T=\infty $ the Bayes' return is maximised by the
\textit{Gittins Index }(GI) policy, see \citet{gittins2011multi}. This operates
by choosing, in the current state, any action $a$, satisfying
\begin{equation}
\nu ^{GI}(\Sigma _{a},n_{a},\gamma )=\max_{1\leq b\leq k}\nu ^{GI}(\Sigma
_{b},n_{b},\gamma )\,, 
\end{equation}%
where $\nu ^{GI}$ is the \textit{Gittins index. }We describe Gittins indices
in Section \ref{sec:newpols} along with versions adapted for use in problems with
$0<\gamma \leq 1,T<\infty $. Given the challenge of computing Gittins indices and the
general intractability of deploying dynamic programming to solve
online problems, the prime interest is in the development of heuristic
policies which are easy to compute and which come close to being return
maximising.

\subsection{The Knowledge Gradient Heuristic}
\label{sec:KG}
The \textit{Knowledge Gradient policy} KG is a heuristic which bases
action choices both on immediate returns $\frac{\Sigma _{a}}{n_{a}}$ and
also on the changes in informational state which flow from a single observed
reward. It is generally fast to compute. To understand how KG works for
MABs suppose that the decision time is $t$ and that the system is in
information state $(\mathbf{\Sigma,n}) $ then. The current
decision is taken to be the last opportunity to learn and so from time $t+1$
through to the end of the horizon whichever arm has the highest mean reward
following the observed reward at $t$ will be pulled at all subsequent times.
With this informational constraint, the best arm to pull at $t$ (and the
action mandated by KG in state $\left( \mathbf{\Sigma ,n}\right) $) is
given by 
\begin{equation}
A^{KG}( \mathbf{\Sigma ,n,}t) =\arg \max_{a}\left\{ \frac{\Sigma
_{a}}{n_{a}}+H( \gamma ,s) \max_{1\leq b\leq k}E\left( \frac{%
\Sigma _{b}+I_{a}Y}{n_{b}+I_{a}}\mid \mathbf{\Sigma ,n,}a\right) \right\} , 
\end{equation}%
where $Y$ is the observed reward at $t$ and $I_{a}$ is an indicator taking
the value $1$ if action $a$ is taken at $t,$ $0$ otherwise. The conditioning
indicates that the reward $Y$ depends upon the current state $\left( \mathbf{%
\Sigma ,n}\right) $ and the choice of action $a$. The constant $H(\gamma,s)$
is a suitable multiplier of the mean return of the best arm at $t+1$ to achieve
an accumulation of rewards for the remainder of the horizon (denoted here by
$s=T-t$). It is given by
\begin{equation}
H( \gamma ,s)=
\begin{cases}
\frac{\gamma(1-\gamma ^{s-1})}{1-\gamma }&\mbox{if }0<\gamma<1,T<\infty\\
\frac{\gamma}{1-\gamma} & \mbox{if } 0<\gamma<1,T=\infty\\
s-1  & \mbox{if }  \gamma =1,T<\infty.
\end{cases}
\end{equation}

KG can be characterised as the policy resulting from the application of a single
policy improvement step to a policy which always pulls an arm
with the highest prior mean return throughout. Note that for $\left( \gamma
,T\right) \in ( 0,1) \times \mathbb{Z}^{+},$ $H( \gamma ,s) $ is increasing in both $\gamma $
($T$ fixed) and in $T$ ($\gamma $ fixed). For any sequence of $\left( \gamma
,T\right) $ values approaching the limit $\left( 1,\infty \right) $ in a manner which is
co-ordinatewise increasing, the value of $H( \gamma ,s) $ diverges
to infinity. This fact is utilised heavily in Section \ref{sec:dom}.

We now develop an equivalent characterisation of KG based on
\citet{ryzhov2012knowledge} which will be more convenient for what follows. We
firstly develop an expression for the change in the maximal mean reward
available from any arm when action $a$ is taken in state $\left( \mathbf{\Sigma
,n}\right) $. We write
\begin{equation}
\nu _{a}^{KG}\left( \mathbf{\Sigma ,n}\right) =E\left\{ \max_{1\leq b\leq
k}\mu _{b}^{+1}-\max_{1\leq b\leq k}\mu _{b}\mid \mathbf{\Sigma ,n,}%
a\right\} , 
\end{equation}%
where $\mu _{b}$ is the current arm $b$ mean return $\frac{\Sigma _{b}}{n_{b}%
}$ and $\mu _{b}^{+1}$ is the mean return available from arm $b$ at the next
time conditional on the observed reward resulting from action $a$. Please
note that $\mu _{b}^{+1}$ is a random variable. It is straightforward to
show that 
\begin{equation}
A^{KG}( \mathbf{\Sigma ,n,}t) =\arg \max_{1\leq a\leq k}\left\{
\mu _{a}+H( \gamma ,s) \nu _{a}^{KG}( \mathbf{\Sigma ,n}%
) \right\}. 
\end{equation}

Hence KG gives a score to each arm and chooses the arm of highest score.
It is not an index policy because the score depends upon the informational
state of arms other than the one being scored. That said, there are
similarities between KG scores and Gittins indices. The Gittins index $\nu
^{GI}(\Sigma _{a},n_{a})$ exceeds the mean return $\frac{\Sigma _{a}}{n_{a}}$
by an amount termed the \textit{uncertainty }or \textit{learning bonus}.
This bonus can be seen as a measure of the value of exploration in choosing
arm $a$. The quantity $H( \gamma ,s) \nu _{a}^{KG}\left( 
\mathbf{\Sigma ,n}\right) $ in the KG score is an alternative estimate of
the learning bonus. Assessing the accuracy of this estimate will give an
indication of the strengths and weaknesses of the policy.

\subsection{Dominated arms}

In our discussion of the deficiencies of the KG policy in the next section
we shall focus, among other things, on its propensity to pull arms which are
suboptimal to another arm with respect to \textit{both} exploitation and
exploration. Hence there is an alternative which is better both from an
immediate return and from an informational perspective. We shall call such
arms \textit{dominated}. We begin our discussion with a result concerning
properties of Gittins indices established by \citet{yu2011structural}.

\begin{theorem}
The Gittins index $\nu ^{GI}\left( c\Sigma ,cn,\gamma \right)$ is decreasing in 
$c\in\mathbb{R}^{+}$ for any fixed $\Sigma ,n,\gamma$ and is increasing
in $\Sigma $ for any fixed $c,n,\gamma $.
\end{theorem}

We proceed to a simple corollary whose proof is omitted. The statement of
the result requires the following definition.

\begin{definition}
An arm in state $\left( \Sigma ,n\right) $ dominates one in state
$\left(\Sigma',n'\right) $ if and only if $\frac{\Sigma
}{n}>\frac{\Sigma'}{n'}$ and $n<n'$.
\end{definition}

\begin{corollary}
The GI policy never chooses dominated arms.
\end{corollary}

Hence pulling dominated arms can never be optimal for infinite horizon MABs.
We shall refer to the pulling of a dominated arm as a \textit{dominated
action }in what follows. Exploration of the conditions under which KG
chooses dominated actions is a route to an understanding of its deficiencies
and prepares us to propose modifications to it which achieve improved
performance. This is the subject matter of the following two sections.

\section{The KG policy and dominated actions}
\label{sec:dom}
\subsection{Conditions for the choice of dominated actions under KG}

This section will elucidate sufficient conditions for the KG policy to
choose dominated arms. A key issue here is that the quantity $\nu _{a}^{KG}$
(and hence the KG learning bonus) can equal zero in cases where the 
\textit{true }learning bonus related to a pull of arm $a$ may be far from
zero. \citet{ryzhov2012knowledge} stated that $\nu _{a}^{KG}>0$. However,
crucially, that paper only considered Gaussian bandits. The next lemma is
fundamental to the succeeding arguments. It says that, for sufficiently
high $\gamma$, the KG policy will choose the arm with the largest $\nu^{KG}$.

\begin{lemma}   \label{lem:dom1}
$\forall \mathbf{\Sigma },\mathbf{n},t$ for which $max_{a}\nu
_{a}^{KG}\left( \mathbf{\Sigma },\mathbf{n}\right) >0$ $\exists \gamma
^{\ast },T^{\ast }$ such that $\gamma >\gamma ^{\ast },T>T^{\ast
}\Rightarrow A^{KG}\left( \mathbf{\Sigma },\mathbf{n},t\right) =\arg
\max_{a}\nu _{a}^{KG}\left( \mathbf{\Sigma },\mathbf{n}\right)$.
\end{lemma}
\begin{proof}
The result is a trivial consequence of the definition of the KG policy in
Section \ref{sec:KG} together with the fact that $H( \gamma ,s) $
diverges to infinity in the manner described in Section \ref{sec:expMAB}.
\end{proof}

The next result gives conditions under which $\nu _{a}^{KG}\left( \mathbf{%
\Sigma },\mathbf{n}\right) =0$. 

\begin{lemma}   \label{lem:dom2}
Let $C_{a}\left( \mathbf{\Sigma},\mathbf{n}\right) $ denote $\max_{b\neq a}\mu _{b}=\max_{b\neq a}%
\frac{\Sigma _{b}}{n_{b}}$. If $a\in \arg \max_{b}\mu _{b}$ and the observation
state space, $\Omega$, is bounded below with minimum value $\min\Omega$ then
\begin{align}   \label{equ:dom1}
\nu _{a}^{KG}\left( \mathbf{\Sigma },\mathbf{n}\right) =0\Leftrightarrow 
\frac{\Sigma _{a}+\min \Omega }{n_{a}+1}\geq C_{a}\left( \mathbf{\Sigma },%
\mathbf{n}\right) ;
\end{align}
while if $a\notin \arg \max_{b}\mu _{b}$ and $\Omega$ is bounded
above with maximum value $\max\Omega$ then 
\begin{align}   \label{equ:dom2}
\nu _{a}^{KG}\left( \mathbf{\Sigma },\mathbf{n}\right) =0\Leftrightarrow 
\frac{\Sigma _{a}+\max \Omega }{n_{a}+1}\leq C_{a}\left( \mathbf{\Sigma },%
\mathbf{n}\right) .
\end{align}
In cases where $a\in \arg \max_{b}\mu _{b}$ with $\Omega$ unbounded below, and
where  $a\notin \arg \max_{b}\mu _{b}$ with $\Omega$ is unbounded
above, we have $\nu _{a}^{KG}\left( \mathbf{\Sigma },\mathbf{n}\right)>0$.
\end{lemma}

\begin{proof}
Note that 
\begin{align}
\nu _{a}^{KG}\left( \mathbf{\Sigma },\mathbf{n}\right)&=E_{Y_{a}}\left[
\max_{b}\mu _{b}^{+1}-\max_{b}\mu _{b}\mid \mathbf{\Sigma },\mathbf{n,}a%
\right]\notag\\
&=E_{Y_{a}}\left[ \max \left( \mu _{a}^{+1},C_{a}\left( \mathbf{\Sigma },%
\mathbf{n}\right) \right) \mid \mathbf{\Sigma },\mathbf{n,}a\right]
-\max_{b}\mu _{b}. 
\end{align}%
Hence 
\begin{align}  \label{equ:dom3}
\nu _{a}^{KG}\left( \mathbf{\Sigma },\mathbf{n}\right) =0\Leftrightarrow
E_{Y_{a}}\left[ \max \left( \mu _{a}^{+1},C_{a}\left( \mathbf{\Sigma },%
\mathbf{n}\right) \right) \mid \mathbf{\Sigma },\mathbf{n,}a\right]
=\max_{b}\mu _{b}. 
\end{align}
If $a\in \arg \max_{b}\mu _{b}$ and so $max_{b}\mu _{b}=\mu _{a}$ then,
observing that 
\begin{equation}
E_{Y_{a}}\left[ \max \left( \mu _{a}^{+1},C_{a}\left( \mathbf{\Sigma },%
\mathbf{n}\right) \right) \mid \mathbf{\Sigma },\mathbf{n,}a\right] \geq
E_{Y_{a}}\left[ \mu _{a}^{+1}\mid \mathbf{\Sigma },\mathbf{n,}a\right] =\mu
_{a}, 
\end{equation}%
we infer from equation Eq. (\ref{equ:dom3}) that 
$\nu_{a}^{KG}\left(\mathbf{\Sigma},\mathbf{n}\right)=0$ if and only if
\begin{equation}
\max \left( \mu _{a}^{+1},C_{a}\left( \mathbf{\Sigma },\mathbf{n}\right)
\right) =\mu _{a}^{+1}\Leftrightarrow \mu _{a}^{+1}\geq C_{a}\left( \mathbf{%
\Sigma },\mathbf{n}\right) 
\end{equation}%
with probability $1$ under the distribution of $Y_{a}$. Under our set up as
described in Section \ref{sec:expMAB}, this condition is equivalent to the right
hand side of Eq. (\ref{equ:dom1}). If $a\notin \arg \max_{b}\mu _{b}$ then $max_{b}\mu
_{b}=C_{a}\left( \mathbf{\Sigma },\mathbf{n}\right) $ and so, suitably
modifying the previous argument, we infer that $\nu
_{a}^{KG}\left( \mathbf{\Sigma },\mathbf{n}\right) =0$ if and only if
\begin{equation}
\max \left( \mu _{a}^{+1},C_{a}\left( \mathbf{\Sigma },\mathbf{n}\right)
\right) =C_{a}\left( \mathbf{\Sigma },\mathbf{n}\right) \Leftrightarrow \mu
_{a}^{+1}\leq C_{a}\left( \mathbf{\Sigma },\mathbf{n}\right) 
\end{equation}%
with probability $1$ under the distribution of $Y_{a}$. Under our set up as
described in Section \ref{sec:expMAB}, this condition is equivalent to the right hand
side of Eq. (\ref{equ:dom2}). The unbounded cases follow directly from the
formula for $\nu _{a}^{KG}(\mathbf{\Sigma },\mathbf{n})$ as the change in
$\mu_a$ due to an observation has no finite limit in the direction(s) of
unboundedness. This completes the proof.
\end{proof}
Informally, $\nu _{a}^{KG}=0$ if no outcome from a pull on arm $a$ will
change which arm has maximal mean value. When $a\in \arg \max_{b}\mu _{b}$
this depends on the lower tail of the distribution of $Y_{a}$ while if $%
a\notin \arg \max_{b}\mu _{b}$ it depends on the upper tail. This asymmetry
is important in what follows.

\begin{theorem}  \label{thm:dom1}
If $\Omega $ is bounded below then there are choices of $\mathbf{\Sigma ,n,}%
\gamma ,T$ for which the KG policy chooses dominated arms.
\end{theorem}

\begin{proof}
If we consider cases for which 
\begin{equation}
\frac{\Sigma _{1}}{n_{1}}>\frac{\Sigma _{2}}{n_{2}},n_{1}<n_{2};\Sigma
_{b}=c\Sigma _{2},n_b=cn_{2},3\leq b\leq k,c\geq 1 
\end{equation}%
then it follows that $\mu_2=\mu_b,\nu_2^{KG}\geq\nu_b^{KG},3\leq b\leq k$, and
all arms except $1$ and $2$ can be ignored in the discussion. We first suppose that $\Omega $ unbounded above. It
follows from Lemma \ref{lem:dom2} that
$\nu_2^{KG}(\mathbf{\Sigma},\mathbf{n})>0$. Since $\min \Omega >-\infty ,$ we
can further choose $(\mathbf{\Sigma},\mathbf{n})$ such that
\begin{equation}
\frac{\Sigma _{1}+\min \Omega }{n_{1}+1}\geq \frac{\Sigma _{2}}{n_{2}}%
=C_{1}(\mathbf{\Sigma },\mathbf{n}) . 
\end{equation}%
From the above result we infer that
$\nu_{1}^{KG}(\mathbf{\Sigma},\mathbf{n})=0$. We now suppose that $\Omega $ is
bounded above, and hence that $\infty >\max \Omega >\min \Omega >-\infty $.
Choose $(\mathbf{\Sigma},\mathbf{n})$ as follows: $\Sigma
_{1}=\max \Omega +2\min \Omega ,n_{1}=3,\Sigma _{2}=\max \Omega +3\min \Omega
,n_{2}=4$. It is trivial that these choices mean that arm $1$ dominates arm $2$. We have that 
\begin{equation}
\frac{\Sigma_{1}+\min\Omega}{n_{1}+1}=\frac{\max\Omega+3\min\Omega}{4}
=\frac{\Sigma _{2}}{n_{2}}=C_{1}( \mathbf{\Sigma },\mathbf{n}) 
\end{equation}
and hence that $\nu _{1}^{KG}(\mathbf{\Sigma},\mathbf{n})=0$.
Further we have that 
\begin{equation}
\frac{\Sigma _{2}+\max \Omega }{n_{2}+1}=\frac{2\max \Omega +3\min \Omega }{5%
}>\frac{\Sigma _{1}}{n_{1}}=C_{2}( \mathbf{\Sigma },\mathbf{n}) 
\end{equation}%
and hence that $\nu _{2}^{KG}( \mathbf{\Sigma },\mathbf{n}) >0$.
In both cases discussed (ie, $\Omega $ bounded and unbounded above) we
conclude from Lemma \ref{lem:dom1} the existence of $t,\gamma ^{\ast },T^{\ast }$
such that $\gamma >\gamma ^{\ast },T>T^{\ast }\Rightarrow A^{KG}( \mathbf{\Sigma },%
\mathbf{n},t) =\arg \max_{a}\nu _{a}^{KG}( \mathbf{\Sigma },%
\mathbf{n}) =2,$ which is a dominated arm, as required. This concludes
the proof.
\end{proof}
Although the part of the above proof dealing with the case in which $\Omega $
is bounded above identifies a specific state in which KG will choose a
dominated arm when $H( \gamma ,s) $ is large enough, it
indicates how such cases may be identified more generally. These occur when
the maximum positive change in the mean of the dominated arm ($\mu
_{2}\rightarrow \mu _{2}^{+1}$) is larger than the maximum negative change
in the mean of the dominating arm ($\mu _{1}\rightarrow \mu _{1}^{+1}$).
This can occur both when the $Y_{a}$ have distributions skewed to the right
and also where the corresponding means are both small, meaning that a large $%
y$ can effect a greater positive change in $\mu _{2}$ than can a small $y$ a
negative change in $\mu _{1}$. A detailed example of this is given for the
Bernoulli MAB in the next section. Similar reasoning suggests that the more
general sufficient condition for KG to choose dominated arms, namely $\nu
_{2}^{KG}\left( \mathbf{\Sigma },\mathbf{n}\right)>\nu
_{1}^{KG}\left( \mathbf{\Sigma },\mathbf{n}\right) $ with arm $2$ dominated,
will hold in cases with $\Omega $ unbounded above if the distribution of $%
Y_{a}$ has an upper tail considerably heavier than its lower tail.

\subsection{Stay-on-the winner rules}
\citet{berry1985bandit} demonstrated that optimal
policies for MABs with Bernoulli rewards and general discount sequences (including all cases
considered here) have a stay-on-the-winner property. If arm $a$ is optimal
at some epoch and a pull of $a$ yields a success ($y_{a}=1$) then arm $a$
continues to be optimal at the next epoch. \citet{yu2011structural} extends this
result to the exponential family considered here in the following way: an optimal arm
continues to be optimal following an observed reward which is sufficiently
large. The next result is an immediate consequence.
\begin{lemma}  \label{lem:dom3}
Suppose that $\Omega $ is bounded above. If arm $a$ is optimal at some epoch
and a pull of $a$ yields a maximal reward ($y_{a}=\max \Omega $) then arm $a$
is optimal at the next epoch.
\end{lemma}
The following result states that the KG policy does not share the
stay-on-the-winner character of optimal policies as described in the
preceding lemma. In its statement we use $\mathbf{e}_a$ for the $k-$vector
whose $ath$ component is $1,$ with zeroes elsewhere.
\begin{proposition}
If $\Omega $ is bounded above and below $\exists $ choices of $\mathbf{%
\Sigma ,n,}t,\gamma ,T$ and $a$ for which $A^{KG}\left( \mathbf{\Sigma },%
\mathbf{n},t\right) =a$, $A^{KG}\left( \mathbf{\Sigma +\max \Omega e}_a,%
\mathbf{n+e}_a,t\right) \neq a$.
\end{proposition}
\begin{proof}
For the reasons outlined in the proof of Theorem \ref{thm:dom1} we may assume without
loss of generality that $k=2$. As in that proof we consider the state $%
\left( \mathbf{\Sigma ,n}\right) $ with $\Sigma _{1}=\max \Omega +2\min
\Omega ,n_{1}=3,\Sigma _{2}=\max \Omega +3\min \Omega ,n_{2}=4$. We suppose
that a pull of arm $2$ yields an observed reward equal to $\max \Omega $.
This takes the process state to $\left( \mathbf{\Sigma +\max \Omega e}_2,%
\mathbf{n+e}_2,t\right) =\left( \mathbf{\Sigma }^{/},\mathbf{n}^{/}\right)
,$ say. We use the dashed notation for quantities associated with this new
state. Observe that $\mu _{2}^{/}>\mu _{1}^{/}$ and hence that $2\in \arg
\max_{b}\mu _{b}^{/}$. We note that 
\begin{equation}
\frac{\Sigma _{2}^{/}+\min \Omega }{n_{2}^{/}+1}=\frac{2\max \Omega +4\min
\Omega }{\text{ }6}=\mu _{1}^{/}=C_{2}\left( \mathbf{\Sigma }^{/},\mathbf{n}%
^{/}\right) , 
\end{equation}%
which implies via Lemma \ref{lem:dom2} that $\nu _{2}^{KG}\left( \mathbf{\Sigma
}^{/},%
\mathbf{n}^{/}\right) =0$. We also have that 
\begin{equation}
\frac{\Sigma _{1}^{/}+\max \Omega }{n_{1}^{/}+1}=\frac{2\max \Omega +2\min
\Omega }{4}>\mu _{2}^{/}=C_{1}\left( \mathbf{\Sigma }^{/},\mathbf{n}%
^{/}\right) , 
\end{equation}%
which implies via Lemma \ref{lem:dom2} that $\nu _{1}^{KG}\left( \mathbf{\Sigma
}^{/},%
\mathbf{n}^{/}\right) >0$. The existence of $t,$ $\gamma ,T$ for which
$A^{KG}\left( \mathbf{\Sigma },\mathbf{n},t\right) =2$ while $A^{KG}\left(
\mathbf{\Sigma}^{/},\mathbf{n}^{/},t+1\right)=A^{KG}\left(\mathbf{\Sigma +\max
\Omega e}_2,\mathbf{n+e}_2,t\right) \neq 2$ now follows from Lemma \ref{lem:dom1}.
\end{proof}
\subsection{Examples}
\label{sec:domexamples}
We will now give more details of how the KG policy chooses dominated
actions in the context of two important members of the exponential family.
\subsubsection{Exponential rewards}
Suppose that $Y_{a}\mid \theta _{a}\backsim Exp( \theta _{a}) $
and $\theta _{a}\backsim Gamma( n_{a}+1,\Sigma _{a}) $ which
yields the unconditional density for $Y_{a}$ given by 
\begin{equation}
g_{a}( y) =( n_{a}+1) \Sigma _{a}^{n_{a}+1}(
\Sigma _{a}+1) ^{-n_{a}-2},y\geq 0, 
\end{equation}%
with $E( Y_{a}) =\frac{\Sigma _{a}}{n_{a}}$. Let arm $1$ dominate
arm $2$. For this case $\Omega =[ 0,\infty ) $ and from Lemma
\ref{lem:dom2}, the unboundedness of $\Omega$ above means that
$\nu_{2}^{KG}(\mathbf{\Sigma ,n})>0$ while $\nu_{1}^{KG}(\mathbf{\Sigma,n})=0$
if and only if
\begin{align}   \label{equ:dom4}
\frac{\Sigma_{1}}{n_{1}+1}\geq \frac{\Sigma _{2}}{n_{2}}. 
\end{align}
Hence from Lemma \ref{lem:dom1} we can assert the existence of $t,\gamma ,T$ for
which KG chooses dominated arm $2$ whenever (Eq. (\ref{equ:dom4}) holds.

\citet{ryzhov2011value} discuss the online KG policy for Exponential
rewards in detail. They observe that $\nu _{a}^{KG}$ can be zero but do not
appear to recognise that this can yield dominated actions under the policy.
Later work,  \cite{ding2016optimal}, showed that this can lead to the
offline KG policy never choosing the greedy arm, an extreme case of dominated
errors. However, with the online KG policy the greedy arm will eventually be
selected as $\nu _{a}^{KG}$ for the other arm tends to zero. These papers note
that, in states for which
\begin{equation}
\frac{\Sigma _{1}}{n_{1}+1}\leq \frac{\Sigma _{2}}{n_{2}}\leq \frac{\Sigma
_{1}}{n_{1}},
\end{equation}%
the value of $\nu _{1}^{KG}\left( \mathbf{\Sigma },\mathbf{n}\right) ,$
while not zero, penalises the choice of the greedy arm relative to other
arms in a similar way to the bias which yields dominated actions. Policies
which mitigate such bias are given in the next section and are evaluated in
the computational study following.

\subsubsection{Bernoulli rewards}
\label{sssec:bern}
Suppose that $Y_{a}\mid \theta _{a}\backsim Bern\left( \theta _{a}\right) ,$
with $\theta _{a}\backsim Beta\left( \Sigma _{a},n_{a}-\Sigma _{a}\right) $
and so $\Omega =\left\{ 0,1\right\} $ and $P\left( Y_{a}=1\right) =\frac{%
\Sigma _{a}}{n_{a}}=1-P\left( Y_{a}=0\right) $. Since $\Omega $ is bounded
above and below, dominated actions under KG will certainly occur.
Demonstrating this in terms of the asymmetric updating of Beta priors can be
helpful in understanding the more general case of bounded rewards. Use $%
\delta _{a}^{+}$ and $\delta _{a}^{-}$ for the magnitudes of the upward and
downward change in $E\left( Y_{a}\right) $ under success and failure
respectively. We have 
\begin{equation}
\delta _{a}^{+}=\frac{n_{a}-\Sigma _{a}}{n_{a}\left( n_{a}+1\right)
}\;;\;\delta _{a}^{-}=\frac{\Sigma _{a}}{n_{a}\left( n_{a}+1\right) }, 
\end{equation}%
from which we conclude that $\delta _{a}^{+}\geq \delta
_{a}^{-}\Leftrightarrow \mu _{a}\leq 0.5$. Prompted by this analysis,
consider a case in which $k=2,\Sigma _{1}=\Sigma _{2};n_{1}+m=n_{2}$ for
some $m\in\mathbb{N}^{+}$. Arm $1$ dominates arm $2$. Further, the fact that 
\begin{equation}
\frac{\Sigma _{1}+\min \Omega }{n_{1}+1}=\frac{\Sigma _{1}}{n_{1}+1}\geq 
\frac{\Sigma _{1}}{n_{1}+m}=\frac{\Sigma _{2}}{n_{2}}=C_{1}\left( \mathbf{%
\Sigma },\mathbf{n}\right) 
\end{equation}
implies via Lemma \ref{lem:dom2} that $\nu _{1}^{KG}\left( \mathbf{\Sigma
},\mathbf{n}%
\right) =0$. From Lemma \ref{lem:dom2} we also conclude that 
\begin{align}   \label{equ:dom5}
\nu _{2}^{KG}\left( \mathbf{\Sigma },\mathbf{n}\right) >0\iff 
\frac{\Sigma _{2}+\max \Omega }{n_{2}+1}=\frac{\Sigma
_{1}+1}{n_{1}+m+1}>\frac{\Sigma _{1}}{n_{1}}=C_{2}\left( \mathbf{\Sigma },\mathbf{n}\right).
\end{align}

The strict inequality in the right hand side of Eq. (\ref{equ:dom5}) will hold
whenever $n_{1}>\left( m+1\right) \Sigma _{1}$. Thus, for suitably chosen
$t,\gamma$ and $T,$ the KG policy will take dominated actions in a wide range of states. Suppose now that $T=\infty $ and hence the immediate claim
is that under the condition $n_{1}>\left( m+1\right) \Sigma _{1}$ the KG
policy will take dominated action $2$ for $\gamma $ large enough. We now
observe that in practice dominated actions can be taken for quite modest
$\gamma$. Returning to the characterisation of the KG policy we infer that in the above example, dominated action $2$ will be chosen whenever
\begin{equation}
n_{1}>\left( m+1\right) \Sigma _{1},\frac{\gamma }{1-\gamma }>\frac{m\left(
n_{1}+m+1\right) }{\left\{ n_{1}-\left( m+1\right) \Sigma _{1}\right\} }. 
\end{equation}%
Such errors will often be costly. Note also that the condition $n_{1}>\left(
m+1\right) \Sigma _{1}$ suggests that dominated actions occur more often
when arms have small mean rewards. This is investigated further in the
computational study following.

\subsubsection{Gaussian rewards}

Here we have $Y_{a}\mid\theta _{a}\backsim N\left(\theta _{a},1\right)$
and $\theta_{a}\backsim N\left(\frac{\Sigma_{a}}{n_{a}},\frac{1}{n_{a}}\right)$.
Hence $\Omega=\mathbb{R}$ is unbounded and if arm $a$ is chosen, the
distribution of $\mu _{a}^{+}$ is symmetric about $\mu _{a}$. In this case the
KG policy does not choose dominated actions and the value of $\nu _{a}^{KG}$ is
always greater for the arm with smaller prior precision $n_{a}$. Despite this
fact, KG can still take poor decisions by underestimating the learning bonus for
the greedy arm. The Gaussian MAB is discussed further in Section \ref{sec:gauss}.

\section{Policies which modify KG to avoid taking dominated actions}
\label{sec:newpols}
In this section we present new policies which are designed to
mitigate the defects of the KG approach elucidated in the previous section. The
performance of these are assessed along with some earlier proposals, in the
numerical study of the next section.

\textbf{Non-dominated KG }(NKG): This proposal modifies
standard KG by prohibiting dominated actions. It achieves this by always
choosing a non-dominated arm with highest KG score. Any
greedy arm is non-dominated and hence one always exists.

\textbf{Positive KG} (PKG): The KG score for a greedy
arm reflects a negative change in its posterior mean while that for
non-greedy arms reflect positive changes. The PKG policy modifies KG
such that for all arms it is positive moves which are registered. It
achieves this by modifying the KG scores for each greedy arm $a$ as
follows: in the computation of the score replace the quantity
$C_{a}\left(\mathbf{\Sigma ,n}\right)=\max_{b\neq a}\mu _{b}$ by the quantity
$C_{a}^{\ast}\left( \mathbf{\Sigma
,n}\right):=2\mu_{a}-C_{a}\left(\mathbf{\Sigma ,n}\right)$. 
This adjustment transforms the KG scores
$\nu_{a}^{KG}\left(\mathbf{\Sigma,n}\right)$ to adjusted values $\nu
_{a}^{PKG}\left(\mathbf{\Sigma,n}\right)$.  The change maintains the key
distance used in the KG calculation as $C^*_a-\mu_a=\mu_a-C_a$ but ensures that
it is non-negative. For non-greedy arms $b$ we have $\nu
_{b}^{KG}\left(\mathbf{\Sigma,n}\right) =\nu _{b}^{PKG}\left(\mathbf{\Sigma ,n}\right)$.  

\begin{theorem} \label{thm:pols1}
Policy PKG never chooses a strictly dominated arm.
\end{theorem}

\begin{proof}
Suppose that arm $2$ is strictly dominated by arm $1$ such that $\frac{%
\Sigma _{1}}{n_{1}}>\frac{\Sigma _{2}}{n_{2}}$ and $n_{2}\geq n_{1}+1$. In
the argument following we shall suppose that $k=2$. This is without loss of
generality as the addition of any other arm $b$ with $\mu_b\leq\mu_1$ does not
effect the PKG score of arm 2 and can only increase the PKG score of the
non-dominated arm 1. Given that $\mu _{1}>\mu _{2},$ in order to establish the
result, namely that $A^{PKG}\left( \mathbf{\Sigma ,n}\right) =1$ it is
enough to establish that $\nu _{1}^{PKG}\left( \mathbf{\Sigma ,n}\right)
\geq \nu _{2}^{PKG}\left( \mathbf{\Sigma ,n}\right) $. From the definitions
of the quantities concerned we have that 
\begin{align}
\nu _{1}^{PKG}\left( \mathbf{\Sigma ,n}\right)&=E\left\{ \max \left( \mu
_{1}^{+}-C_{1}^{\ast }\left( \mathbf{\Sigma ,n}\right) \mid \mathbf{\Sigma
,n,}1\right) ,0\right\}\notag\\
&=E_{Y_{1}}\max \left\{ \left( \frac{\Sigma _{1}+Y_{1}}{n_{1}+1}-\left( \frac{%
2\Sigma _{1}}{n_{1}}-\frac{\Sigma _{2}}{n_{2}}\right) \right) ,0\right\} ,
\end{align}%
while%
\begin{equation}
\nu _{2}^{PKG}\left( \mathbf{\Sigma ,n}\right) =E_{Y_{2}}\max \left\{
\left( \frac{\Sigma _{2}+Y_{2}}{n_{2}+1}-\frac{\Sigma _{1}}{n_{1}}\right)
,0\right\} .
\end{equation}%
However, under the conditions satisfied by $\left( \mathbf{\Sigma ,n}\right) 
$ it is easy to show that, $\forall y\in\mathbb{R}$,
\begin{equation}
\max \left\{ \left( \frac{\Sigma _{1}+y}{n_{1}+1}-\left( \frac{2\Sigma _{1}}{%
n_{1}}-\frac{\Sigma _{2}}{n_{2}}\right) \right) ,0\right\} \geq \max \left\{
\left( \frac{\Sigma _{2}+y}{n_{2}+1}-\frac{\Sigma _{1}}{n_{1}}\right)
,0\right\} 
\end{equation}%
and hence that 
\begin{equation}
\nu _{1}^{PKG}\left( \mathbf{\Sigma ,n}\right) \geq E_{Y_{1}}\max \left\{
\left( \frac{\Sigma _{2}+Y_{1}}{n_{2}+1}-\frac{\Sigma _{1}}{n_{1}}\right)
,0\right\} .
\end{equation}%
But from \citet{shaked2007stochastic} we infer that $Y_{1}$ exceeds $Y_{2}$ in
the convex ordering. Since $\max \left\{ \left( \frac{\Sigma _{2}+y}{%
n_{2}+1}-\frac{\Sigma _{1}}{n_{1}}\right) ,0\right\} $ is convex in $y$ it
follows that
\begin{align}
\nu _{1}^{PKG}\left( \mathbf{\Sigma ,n}\right) &\geq E_{Y_{1}}\max \left\{
\left( \frac{\Sigma _{2}+Y_{1}}{n_{2}+1}-\frac{\Sigma _{1}}{n_{1}}\right)
,0\right\}\notag\\ 
&\geq
E_{Y_{2}}\max\left\{\left(\frac{\Sigma_{2}+Y_{2}}{n_{2}+1}-\frac{\Sigma_{1}}{n_{1}}\right),0\right\}\notag\\
&=\nu_{2}^{PKG}\left(\mathbf{\Sigma
,n}\right)
\end{align} 
and the result follows.
\end{proof}

\textbf{KG-index} (KGI): Before we describe this
proposal we note that \citet{whittle1988restless} produced a proposal for index
policies for a class of decision problems called \textit{restless bandits }which
generalise MABs by permitting movement in the states of non-active arms.
Whittle's indices generalise those of Gittins in that they are equal to the
latter for MABs with $0<\gamma <1,T=\infty $. Whittle's proposal is relevant for
MABs with finite horizon $T<\infty $ since time-to-go now needs to be
incorporated into state information which in turn induces a form of
restlessness. In what follows we shall refer to Gittins/Whittle indices as those
which emerge from this body of work for all versions of the MABs under
consideration here.

The KGI policy chooses between arms on the basis of an index which
approximates the Gittins/Whittle index appropriate for the problem by using
the KG approach. We consider a single arm with $\left( \Sigma ,n\right) $
prior, finite horizon $t$ and discount factor $\gamma ,0\leq \gamma \leq 1$.
To develop the Gittins/Whittle index $\nu _{t}^{GI}\left( \Sigma ,n,\gamma
\right) $ for such a bandit we suppose that a charge $\lambda $ is levied
for bandit activation. We then consider the sequential decision problem
which chooses from the actions $\left\{ active,passive\right\} $ for the
bandit at each epoch over horizon $t$ with a view to maximising expected
rewards net of charges for bandit activation. The value function $%
V_{t}\left( \Sigma ,n,\gamma ,\lambda \right) $ satisfies Bellman's
equations as follows:
\begin{align}   \label{equ:kgi1}
V_{t}\left(\Sigma ,n,\gamma,\lambda \right)=\max\left\{ \frac{\Sigma }{n}-\lambda +\gamma E_{Y}\left[ V_{t-1}\left( \Sigma +Y,n+1,\gamma ,\lambda
\right) \right] ;V_{t-1}\left( \Sigma ,n,\gamma ,\lambda \right) \right\}. 
\end{align}
It is easy to show that this is a stopping problem in that, once it is
optimal to choose the passive action at some epoch then it will be optimal
to choose the passive action at all subsequent epochs. Hence, Eq.
(\ref{equ:kgi1}) may be replaced by the following:%
\begin{equation}
V_{t}\left( \Sigma ,n,\gamma ,\lambda \right) =\max \left\{ \frac{\Sigma }{n}%
-\lambda +\gamma E_{Y}\left[ V_{t-1}\left( \Sigma +Y,n+1,\gamma ,\lambda
\right) \right] ;0\right\} . 
\end{equation}%
We further observe that $V_{t}\left( \Sigma ,n,\gamma ,\lambda \right) $
decreases as $\lambda $ increases, while keeping $t,\Sigma ,n$ and $\gamma $
fixed. This yields the notion of \textit{indexability }in index theory. We
now define the Gittins/Whittle index as%
\begin{equation}
\nu _{t}^{GI}\left( \Sigma ,n,\gamma \right) =\min \left\{ \lambda
;V_{t}\left( \Sigma ,n,\gamma ,\lambda \right) =0\right\} . 
\end{equation}%
This index is typically challenging to compute.

We obtain an index approximation based on the KG approach as follows: In
the stopping problem with value function $V_{t}\left( \Sigma ,n,\gamma
,\lambda \right) $ above, we impose the constraint that whatever decision is
made at the second epoch is final, namely will apply for the remainder of
the horizon. This in turn yields an approximating value function $%
V_{t}^{KG}\left( \Sigma ,n,\gamma ,\lambda \right) $ which when $0<\gamma <1$
satisfies the equation
\begin{align}   \label{equ:kgi2}
&V_{t}^{KG}\left( \Sigma ,n,\gamma ,\lambda \right)\notag\\
&=\max \left\{ \frac{\Sigma }{n}-\lambda +\frac{\gamma \left( 1-\gamma ^{t-1}\right) }{\left(
1-\gamma \right) }E_{Y}\left[ \max \left( \max \left( \frac{\Sigma +Y}{n+1}%
,\lambda \right) -\lambda ;0\right) \mid \Sigma ,n\right] ;0\right\}
\end{align}
and which is also decreasing in $\lambda $ for any fixed $t,\Sigma ,n$ and $%
\gamma $. When $\gamma =1$ the constant multiplying the expectation on the
r.h.s of Eq. (\ref{equ:kgi2}) becomes $t-1$. The indices we use for the
KGI policy when $T<\infty $ are given by
\begin{align}
\nu _{t}^{KGI}\left( \Sigma ,n,\gamma \right)&=\min \left\{ \lambda
;V_{t}^{KG}\left( \Sigma ,n,\gamma ,\lambda \right) =0\right\}\notag\\
&=\min \left\{
\lambda ;\lambda \geq \frac{\Sigma }{n}\text{ and\ }V_{t}^{KG}\left( \Sigma
,n,\gamma ,\lambda \right) =0\right\},
\label{equ:kgi3}
\end{align}
where $\Sigma ,n,\gamma $ are as previously and $t$ is the time to the end
of the horizon. Note that the second equation in Eq. (\ref{equ:kgi3}) follows
from the evident fact that the index is guaranteed to be no smaller that the
mean $\frac{\Sigma }{n}$.

Trivially $V_{t}\left( \Sigma ,n,\gamma ,\lambda\right)$ and $V_{t}^{KG}\left(
\Sigma ,n,\gamma ,\lambda \right) $ are both increasing in the horizon $t$ and
consequentially so are both $\nu_{t}^{GI}\left( \Sigma ,n,\gamma \right) $ and
$\nu _{t}^{KGI}\left( \Sigma ,n,\gamma \right) $. When $0<\gamma <1$ the limits $%
\lim_{t\rightarrow \infty }\nu _{t}^{GI}\left( \Sigma ,n,\gamma \right) $
and $\lim_{t\rightarrow \infty }\nu _{t}^{KGI}\left( \Sigma ,n,\gamma\right)$
are guaranteed to exist and be finite. These limits are denoted $%
\nu ^{GI}\left( \Sigma ,n,\gamma \right) $ and $\nu ^{KGI}\left( \Sigma
,n,\gamma \right) $ respectively, the former being the Gittins index. We use
the indices $\nu ^{KGI}\left( \Sigma ,n,\gamma \right) $ for the KGI
policy when $0<\gamma <1,T=\infty $.

\begin{theorem} \label{thm:pols2}
The KGI policy does not choose dominated arms.
\end{theorem}

We establish this result via a series of results.

\begin{lemma}
$V_{t}^{KG}\left( \Sigma ,n,\gamma ,\lambda \right) $ and $\nu
_{t}^{KGI}\left( \Sigma ,n,\gamma \right) $ are both increasing in $\Sigma 
$\bigskip\ for any fixed values of $t,n,\gamma ,\lambda $.
\end{lemma}
\begin{proof}
Since the quantity $\left( \max \left( \frac{\Sigma +y}{n+1},\lambda \right)
-\lambda ;0\right) $ is increasing in $y$ and $Y\mid \Sigma ,n$ is
stochastically increasing in $\Sigma ,$ it follows easily that the
expectation on the right hand side of Eq. (\ref{equ:kgi2}) is increasing in $%
\Sigma $. The result then follows straightforwardly.
\end{proof}
We now proceed to consider the equivalent bandit, but with prior $\left(
c\Sigma ,cn\right) ,$ where $c>0$.
\begin{lemma}
$V_{t}^{KG}\left( c\Sigma ,cn,\gamma ,\lambda \right) $ is decreasing in $c$
for any fixed values of $t,\Sigma ,n,\gamma $ and for any $\lambda \geq 
\frac{\Sigma }{n}$.
\end{lemma}
\begin{proof}
First note that for $y\geq \frac{\Sigma }{n},$ the quantity $\max \left( 
\frac{c\Sigma +y}{cn+1},\lambda \right) ,$ regarded as a function of $c,$ is
decreasing when $\lambda \geq \frac{\Sigma }{n}$. For $y<\frac{\Sigma }{n}%
,\max \left( \frac{c\Sigma +y}{cn+1},\lambda \right) =\lambda $ and hence is
trivially decreasing in $c$. Note also that the quantity $\max \left( \frac{%
c\Sigma +y}{cn+1},\lambda \right) ,$ regarded as a function of $y,$ is
increasing and convex. We also observe from \citet{yu2011structural} that $Y\mid
c\Sigma ,cn $ is decreasing in the convex order as $c$ increases. It then follows
that, for $c_{1}>c_{2}$ and for $\lambda \geq \frac{\Sigma }{n},$ 
\begin{align}
E_{Y}\left( \max \left( \frac{c_{1}\Sigma +Y}{c_{1}n+1},\lambda \right) \mid
c_{1}\Sigma ,c_{1}n\right)&\leq E_{Y}\left( \max \left( \frac{c_{1}\Sigma +Y%
}{c_{1}n+1},\lambda \right) \mid c_{2}\Sigma ,c_{2}n\right)\notag\\
&\leq E_{Y}\left( \max \left( \frac{c_{2}\Sigma +Y}{c_{2}n+1},\lambda \right)
\mid c_{2}\Sigma ,c_{2}n\right) 
\end{align}%
from which the result trivially follows via a suitable form of Eq. (\ref{equ:kgi2}).
\end{proof}
The following is an immediate consequence of the preceding lemma and
Eq. (\ref{equ:kgi3}).
\begin{corollary}
$\nu _{t}^{KGI}(c\Sigma,cn,\gamma)$ is decreasing in $c$ for any fixed
values of $t,\Sigma ,n,\gamma $.
\end{corollary}
It now follows trivially from the properties of the index $\nu _{t}^{KGI}$
established above that if $(\Sigma_{1},n_{1})$ dominates $(\Sigma_{2},n_{2})$
then
$\nu_{t}^{KGI}(\Sigma_{1},n_{1},\gamma)\geq\nu_{t}^{KGI}(\Sigma_{2},n_{2},\gamma)$
for any $t,\gamma $. It must also follow that
$\nu^{KGI}(\Sigma_{1},n_{1},\gamma)\geq\nu ^{KGI}(\Sigma_{2},n_{2},\gamma)$ when
$0<\gamma <1$. This completes the proof of the above theorem.

Closed form expressions for the indices $\nu _{t}^{KGI}$ are not usually
available, but are in simple cases. For the Bernoulli rewards case of
Subsection \ref{sssec:bern} we have that 
\begin{equation}
\nu _{t}^{KGI}\left( \Sigma ,n,\gamma \right) =\frac{\Sigma }{n}+\frac{%
\gamma \left( 1-\gamma ^{t-1}\right) }{\left( 1-\gamma \right) }\frac{\Sigma
\left( \Sigma +1\right) }{\left( n+1\right) \left\{ n+\frac{\gamma \left(
1-\gamma ^{t-1}\right) }{\left( 1-\gamma \right) }\Sigma \right\} }. 
\end{equation}%
In general numerical methods such as bisection are required to obtain the
indices. If the state space is finite it is recommended that all index
values are calculated in advance.

Fast calculation is an essential feature of KG but it should be noted that this
is not universal and that index methods are more tractable in general. An
example of this is the MAB with multiple plays (\citet{whittle1980multi}).
Here $m$ arms are chosen at each time rather than just one. Rewards are received
from each of the arms as normal. For an index policy the computation required is
unchanged - the index must be calculated for each arm as normal with arms chosen
in order of descending indices. The computation for KG is considerably larger
than when $m=1$. The KG score must be calculated for each possible
\emph{combination} of $m$ arms, that is $\binom{n}{m}$ times. For each of these
we must find the set of arms with largest expected reward conditional on each
possible outcome. Even in the simplest case, with Bernoulli rewards, there are
$2^m$ possible outcomes. For continuous rewards the problem becomes much more
difficult even for $m=2$. It is clear that KG is impractical for this problem.

An existing method with similarities to KG is the Expected Improvement
algorithm of \cite{jones1998efficient}. This is an offline method of which KG
can be thought of as a more detailed alternative. It was compared with KG in
\cite{frazier2009knowledge} in the offline setting. The Expected Improvement
algorithm is simpler than KG and always assigns positive value to the greedy
arm unless its true value is known exactly. Its arm values are ``optimistic'' in a
manner analogous to the PKG policy described above and it is reasonable
to conjecture that it shares that rule's avoidance of dominated actions (see
Theorem \ref{thm:pols1}). As an offline method it is not tested here but it may
be possible to develop an online version.

\section{Computational Study}
\label{sec:comp1}
This section will present the results of experimental studies for the Bernoulli
and Exponential MAB. A further study will be made for the Gaussian MAB in
Section \ref{sec:comp2}.
\subsection{Methodology}
All experiments use the standard MAB setup as described in Section \ref{sec:MAB}.
For Bernoulli rewards with $k=2$ policy returns are calculated using
value iteration. All other experiments use simulation for this purpose. These
are \emph{truth-from-prior} experiments i.e. the priors assigned to each arm are
assumed to be accurate.

For each simulation run a $\theta_a$ is drawn randomly from the prior for each
arm $a\in\{1,2,\ldots,k\}$. A bandit problem is run for each policy to be tested
using the same set of parameter values for each policy. Performance is measured
by totalling, for each policy, the discounted true expected reward of the arms
chosen. For each problem 160000 simulation runs were made.

In addition to the policies outlined in Section \ref{sec:newpols}, also tested
are the Greedy policy (described in Section \ref{sec:MAB}) and a policy based on analytical
approximations to the GI (\citet{brezzi2002optimal}), referred to here as GIBL. These approximations
 are based on the GI for a Wiener process and therefore assume Normally
 distributed rewards. However, they can be appropriate for other reward
 distributions by Central Limit Theorem arguments and the authors found that
 the approximation was reasonable for Bernoulli rewards, at least for $n$ not
 too small. Other papers have refined these approximations but, although they
 may be more accurate asymptotically, for the discount rates tested here they
 showed inferior performance and so only results for GIBL are given.
\subsection{Bernoulli MAB}
The first experiment tests performance over a range of $\gamma$ for
$k\in\{2,10\}$ arms, each with uniform $Beta(1,1)$ priors.
The mean percentage lost reward for five policies are given in Figure
\ref{fig:bmab_gamma}. The results for the
greedy policy are not plotted as they are clearly worse than the other policies
(percentage loss going from 0.64 to 1.77 for $k=2$).
\begin{figure}
\centering
\includegraphics[width=\textwidth]{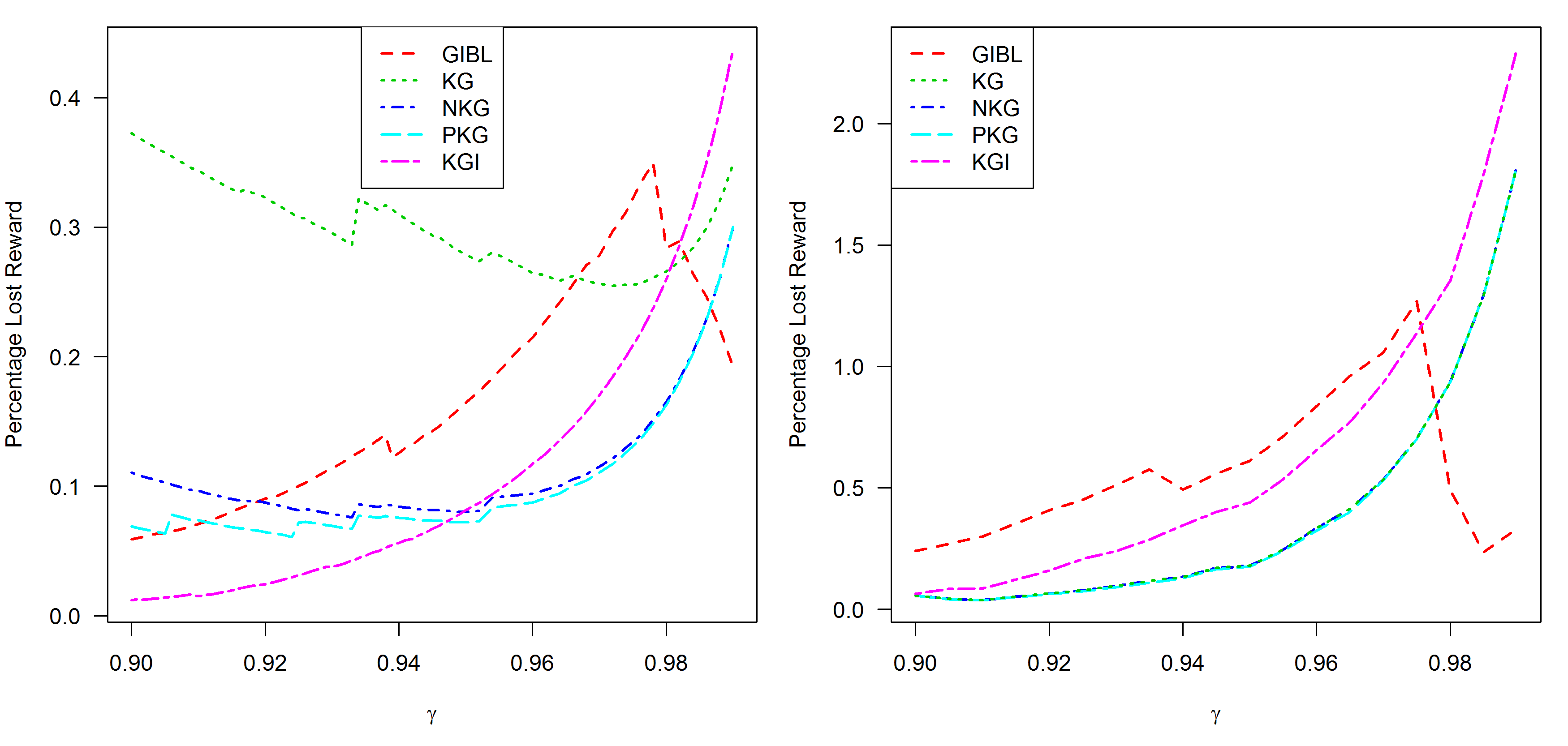}
\caption{Mean percentage of lost reward compared to the GI policy for five
policies for the Bernoulli MAB with uniform priors and $\gamma\in[0.9,0.99]$.
The left plot shows $k=2$ while on the right $k=10$.}
\label{fig:bmab_gamma}
\end{figure}
The overall behaviour of the policies is similar for $k=2$ and $k=10$. KGI is
strong for lower $\gamma$ but is weaker for higher $\gamma$ while GIBL is
strongest as $\gamma$ increases. The sharp change in performance for GIBL at
$\gamma\approx0.975$ occurs because the GIBL index is a piecewise function. Both
NKG and PKG improve on KG for $k=2$ but the three KG variants are almost identical for
$k=10$. The difference between KG and NKG gives the cost for the KG policy of
dominated actions. These make up a large proportion of the lost reward for KG
for lower $\gamma$ but, as $\gamma$ increases, over-greedy errors due to the
myopic nature of the KG policy become more significant and these are not
corrected by NKG. These errors are also the cause of the deteriorating
performance of KGI at higher $\gamma$. At $k=10$ the states given in Section
\ref{sec:domexamples} where KG was shown to take dominated actions occur
infrequently. This is because, for larger numbers of arms there will more often
be an arm with $\mu\geq0.5$ and such arms are chosen in preference to dominated
arms.

However, states where $\mu<0.5$ for all arms will occur more frequently when
arms have lower $\theta$. Here dominated actions can be expected to be more
common. We can test this by using priors where $\beta>\alpha$. Figure
\ref{fig:bmab_beta} shows the effect of varying the $\beta$ parameter for all
arms.
\begin{figure}
\centering
\includegraphics[width=\textwidth]{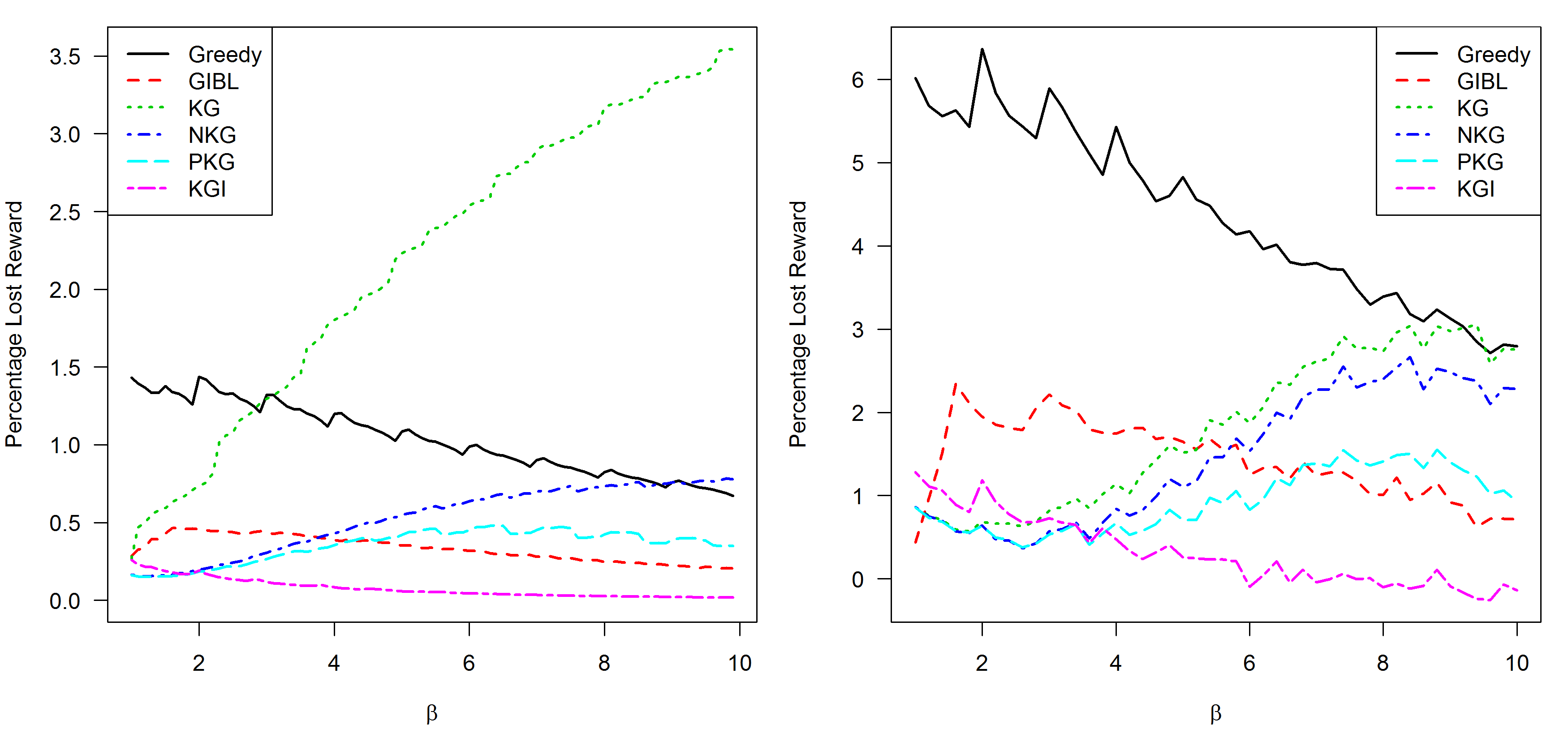}
\caption{Percentage lost reward relative to the GI policy for six policies for
the Bernoulli MAB with $\alpha=1,\beta\in[1,10]$ and $\gamma=0.98$. The
left plot shows $k=2$ while on the right $k=10$.}
\label{fig:bmab_beta}
\end{figure}
The discount rate $\gamma=0.98$ is quite a high value where the greedy policy
can be expected to perform poorly since exploration will be important. However 
as $\beta$ increases the performance of KG deteriorates to the extent that it is
outperformed by the greedy policy. This effect is still seen when $k=10$. The
superior performance of NKG shows that much of the loss of KG is due to
dominated actions. Policy PKG improves further on NKG suggesting that KG makes
further errors due to asymmetric updating even when it does not choose dominated
arms. A clearer example of this is given in Section \ref{sec:emab}. Both
policies based on GI approximations perform well and are robust to changes in
$\beta$. KGI is the stronger of the two as GIBL is weaker when
the rewards are less Normally distributed. 

The same pattern can also be seen to be present when arms have low success
probabilities but prior variance is high. Figure \ref{fig:bmab_alpha} gives
results for $\beta=1$ with low $\alpha$. The range shown focuses on lower
prior $\mu$ which correspond to $\beta\in[2,\ldots50]$ in the setup of
the previous experiment. The higher prior variance makes arms with higher
success probabilities more likely than in the previous experiment but as
$\alpha$ is reduced the performance of KG can still be seen to deteriorate
markedly. The other policies tested do not show this problem.
\begin{figure}
\centering
\includegraphics[width=\textwidth]{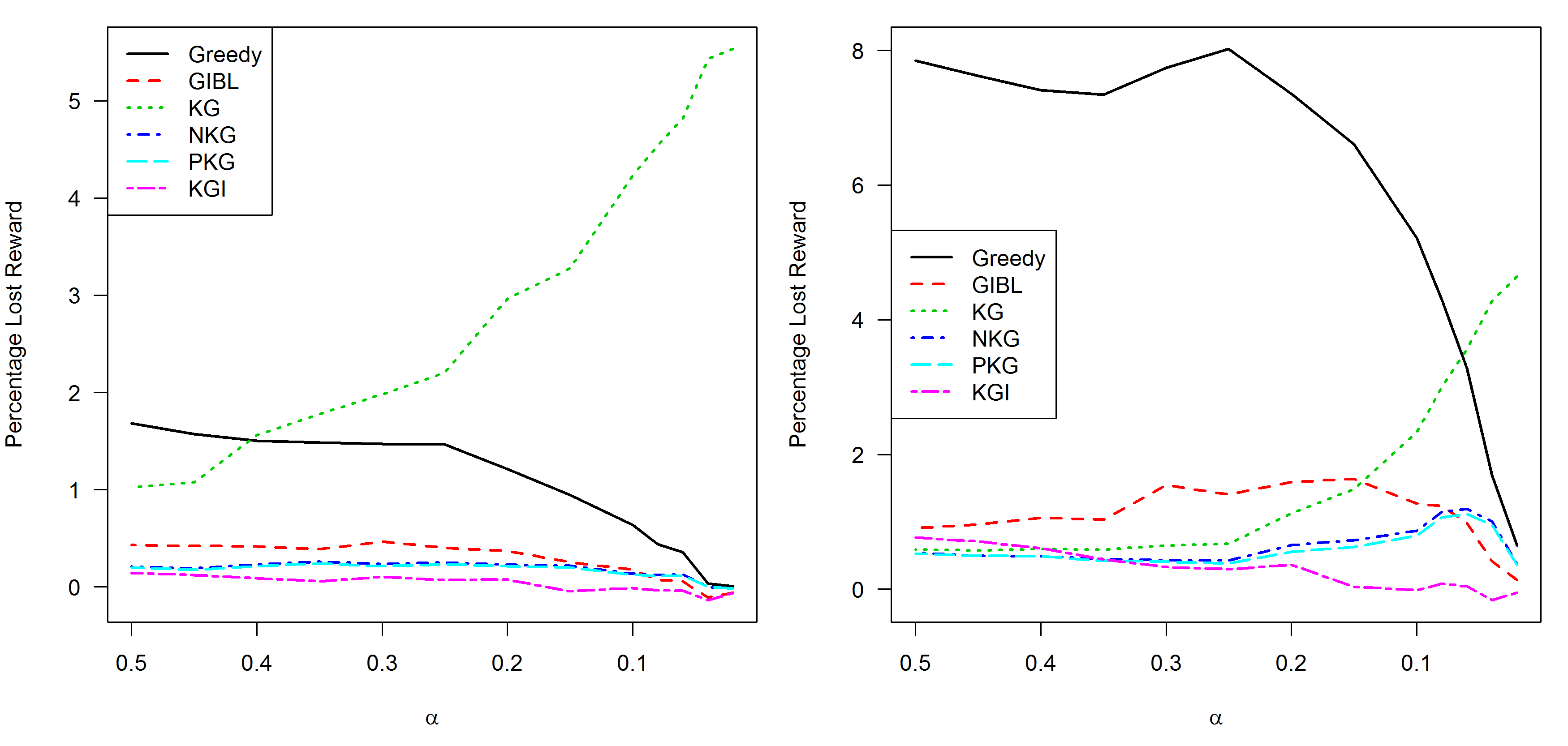}
\caption{Percentage lost reward relative to the GI policy for six policies for
the Bernoulli MAB with $\beta=1,\alpha\in[0.02,0.5]$ and $\gamma=0.98$. The
left plot shows $k=2$ while on the right $k=10$.}
\label{fig:bmab_alpha}
\end{figure}

Arms with low $\theta$ are common in many applications. For
example, in direct mail marketing or web based advertising where
$\theta$ is the probability that a user responds to an advert. The unmodified KG
is unlikely to be an effective method in such cases.

The equivalent plots with prior $\mu>1$ do not show any significant
changes in behaviour compared to uniform priors.

 Another policy that is popular in the bandit literature and which
 has good theoretical properties is Thompson Sampling
 (e.g. \citet{russo2014learning}). Results for this method are not given in
 detail here as its performance is far inferior on these problems to the other policies
 tested. For example, on the results displayed in Figure \ref{fig:bmab_gamma}
 losses were in the ranges from $1.3-4\%$ and $6-15\%$ for $k=2$ and $k=10$
 respectively with the best performance coming for $\gamma=0.99$.  It is a stochastic policy and so makes
 many decisions that are suboptimal (including dominated errors). Its strength
 is that it explores well in the limit over time, eventually finding the true
 best arm. However, with discounted rewards or when the horizon is finite
 it gives up too much short term reward to be competitive unless $\gamma$ is
 close to 1 or the finite horizon is long. In addition, note that it will spend
 longer exploring as $k$ increases as it seeks to explore every alternative.
 Performance on the other problems in this paper was similar and so are
 not given.
\subsection{Exponential MAB}
\label{sec:emab}
This section gives the results of simulations for policies run on the MAB with
Exponentially distributed rewards as outlined in Section \ref{sec:domexamples}.
These are shown in Figure \ref{fig:emab}. Here the lost reward is given relative
to the KG policy (the negative values indicate that the other policies
outperformed KG).
Different priors give a similar pattern of results.
\begin{figure}
\centering
\includegraphics[width=\textwidth]{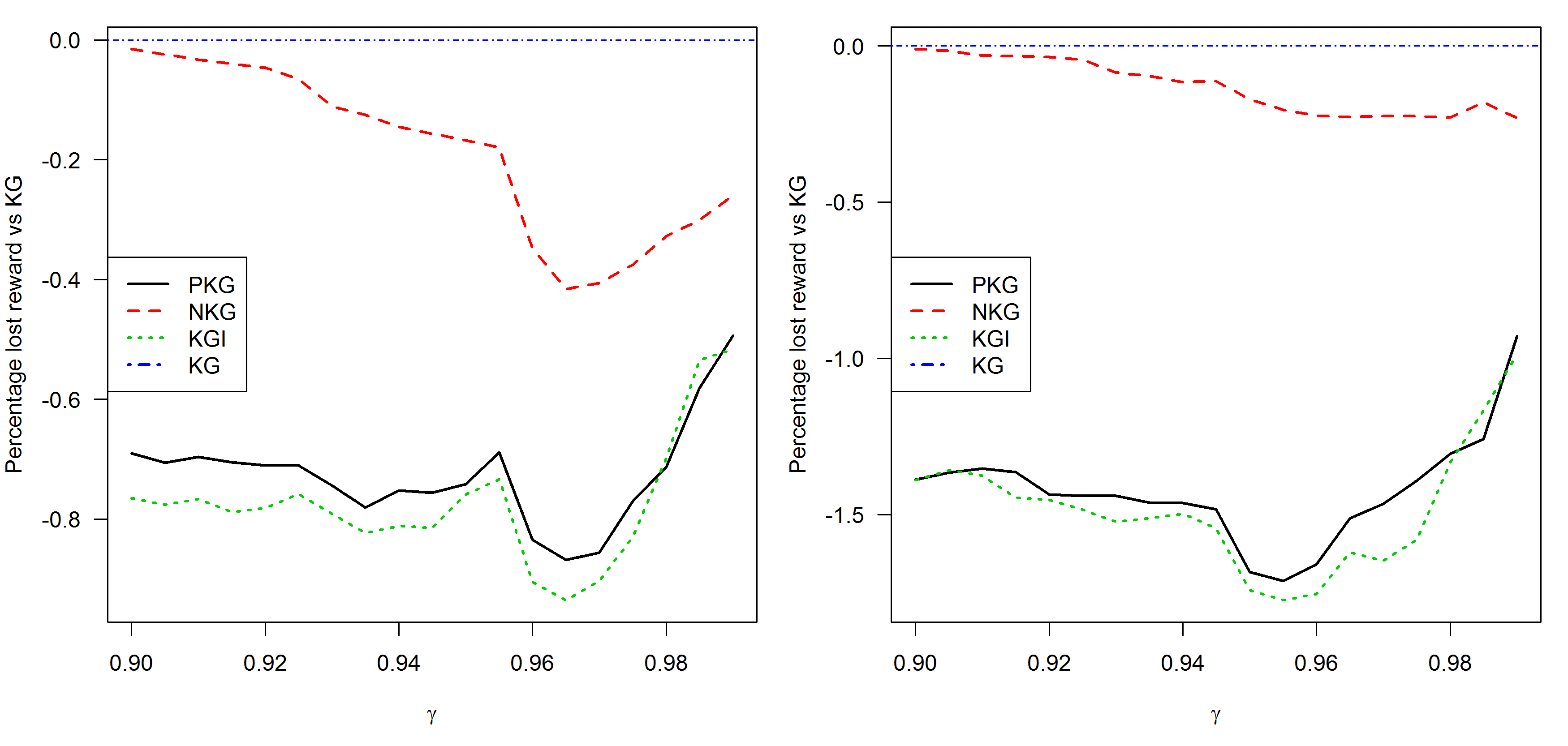}
\caption{Mean percentage of lost reward compared to the KG policy for three
policies for the Exponential MAB with Gamma(2,3) priors and
$\gamma\in[0.9,0.99]$.
The left plot shows $k=2$ while on the right $k=10$.}
\label{fig:emab}
\end{figure}

The results show a clear improvement over the KG policy by PKG and NKG
policies. Notably the PKG earns better reward than the NKG indicating that the
bias that causes dominated errors also causes suboptimal choices when arms
are not dominated. Policy KGI gives the best performance although similar to
PKG.

\section{The Gaussian MAB}
\label{sec:gauss}
Here we consider the Gaussian case $Y_{a}\mid \theta _{a}\backsim N(
\theta _{a},1) $ and $\theta _{a}\backsim N\left( \frac{\Sigma _{a}}{%
n_{a}},\frac{1}{n_{a}}\right) $. In the brief discussion in Section \ref{sec:dom} we
noted that KG does not take dominated actions in this case. While
\citet{ryzhov2012knowledge} give computational results which demonstrate that KG
outperforms a range of heuristic policies, the policy still makes errors. In
this section we describe how errors in the estimation of arms' learning bonuses
constitute a new source of suboptimal actions. We also elucidate easily
computed heuristics which outperform KG. A major advantage of KG cited by
\citet{ryzhov2012knowledge} is its ability to incorporate correlated beliefs
between arms. We will later show, in Section \ref{ssec:cnamb}, that it is
unclear whether KG enjoys a performance advantage in such cases.

We shall restrict the discussion to cases with $k=2$, $0<\gamma <1$ and $%
T=\infty $ and will develop a notion of \textit{relative learning bonus} 
\textit{(RLB) }which will apply across a wide range of policies for such
problems. We shall consider \textit{stationary policies }$\pi $ whose action
in state $\left( \mathbf{\Sigma ,n}\right) \equiv \left( \Sigma
_{1},n_{1},\Sigma _{2},n_{2}\right) $ depends only upon the precisions $n_{b}$
and the difference in means $\Delta \mu :=\frac{\Sigma
_{2}}{n_{2}}-\frac{\Sigma_{1}}{n_{1}}$. We shall write $\pi \left( \Delta \mu
,n_{1},n_{2}\right) $ in what follows. We further require that policies be
\textit{monotone} in the sense of the following definition of the RLB.

\begin{definition}[Relative Learning Bonus]
If $\pi $ is monotone in $\Delta \mu $ such that $\exists $ function $R^{\pi
}:\mathbb{N}^{2}\rightarrow\mathbb{R}$ with
$\pi\left(\Delta\mu,n_{1},n_{2}\right)=2\Leftrightarrow\Delta\mu\geq
R^{\pi}\left(n_{1},n_{2}\right)$ $\forall\left(n_{1},n_{2}\right)$ then $R^{\pi }$ is the RLB function.
\end{definition}

This monotonicity is a natural property of deterministic policies and
holds for all policies considered in this section since increasing $\Delta\mu$
while holding $n_1,n_2$ unchanged favours arm 2 in all cases. The RLB gives a
method of comparing the actions of index and non-index policies but it is also
useful when comparing index policies. A natural method of evaluating an index
policy would be to measure the difference in its indices from GI in the same
states. This can be inaccurate. An index formed by adding a constant to GI will
give an optimal policy so it is not the magnitude of the bias that is important
but how it varies. The RLB and the idea of index consistency (discussed later)
give methods to assess this distinction.

Under the above definition we can set $\Sigma _{1}=0$ without loss of
generality. We then have that $\Delta \mu =\mu _{2}$ and arm $2$ is chosen by
policy $\pi $ in state $\left( \mathbf{\Sigma ,n}\right) $ if and only if $\mu
_{2}\geq R^{\pi }\left( n_{1},n_{2}\right) $. Figure \ref{fig:RLB1} illustrates
this for the GI and KG policies, the former of which determines the optimal
\textit{RLB} values. The plots are of slices through the $R^{GI}$ and $R^{KG}$
surfaces with $n_{1}=1$ and with $\gamma $ set at $0.95$. As $n_{2}$ increases
the GI learning bonus for arm $2$ decreases, yielding values of
$R^{GI}\left(1,n_{2}\right) $ which are increasing and concave. Comparison with
the $R^{KG}\left(1,n_{2}\right)$ suggests that the latter is insufficiently
sensitive to the value of $n_{2}$. This is caused by a KG value close to zero
for arm $2$ when $n_{2}\geq 2$ and results in a mixture of over-exploration and
over-exploitation. In practice when the priors of the two arms are close,
over-exploration is the main problem. For $n_1>1$ the \textit{RLB} curves
have a similar shape but with smaller $R$ as the learning bonuses for
both policies decrease with increased information over time.
\begin{figure}
\centering
\includegraphics[width=\textwidth]{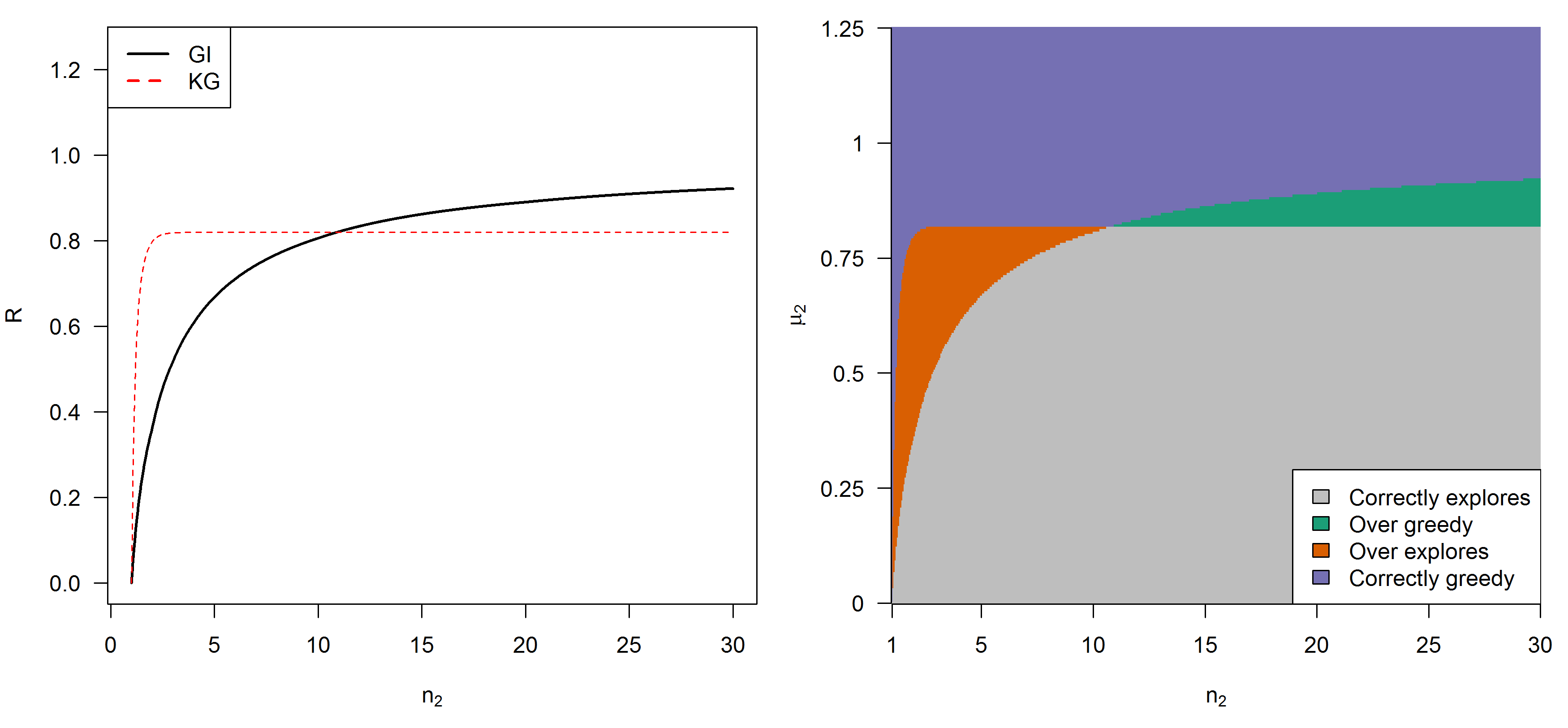}
\caption{The left plot shows RLB values for KG and GI policies for
$\gamma=0.95$, $n_1=1$, $\mu_1=0$. The right plot shows the
nature of KG actions over different arm 2 states.}
\label{fig:RLB1}
\end{figure}

Figure \ref{fig:RLB2} contains comparative plots of $R^{GI}\left( 1,n_{2}\right)
$ and $R^{\pi }\left( 1,n_{2}\right) $ for three other policies $\pi $ and with
$\gamma $ again set at $0.95$. The policies are KGI, described in Section
\ref{sec:newpols}, and two others which utilise analytical approximations to the
Gittins Index, namely GIBL (\citet{brezzi2002optimal}) and GICG
(\citet{chick2009economic}). Although the latter use similar approaches to
approximating GI their behaviour appear quite different, with GIBL
over-greedy and GICG over-exploring. This changes when
$\gamma $ is increased to $0.99$ where both policies over-explore. Although not
shown here, the approximation of GI by both GIBL and GICG improve as
$n_{1}$ increases and the corresponding \textit{RLB} curves are closer. A
suboptimal action is often less costly when over-greedy, especially for lower
$\gamma $ since immediate rewards are guaranteed while the extra information
from exploration might not yield any reward bonus until discounting has reduced
its value.
\begin{figure}
\centering
\includegraphics[width=\textwidth]{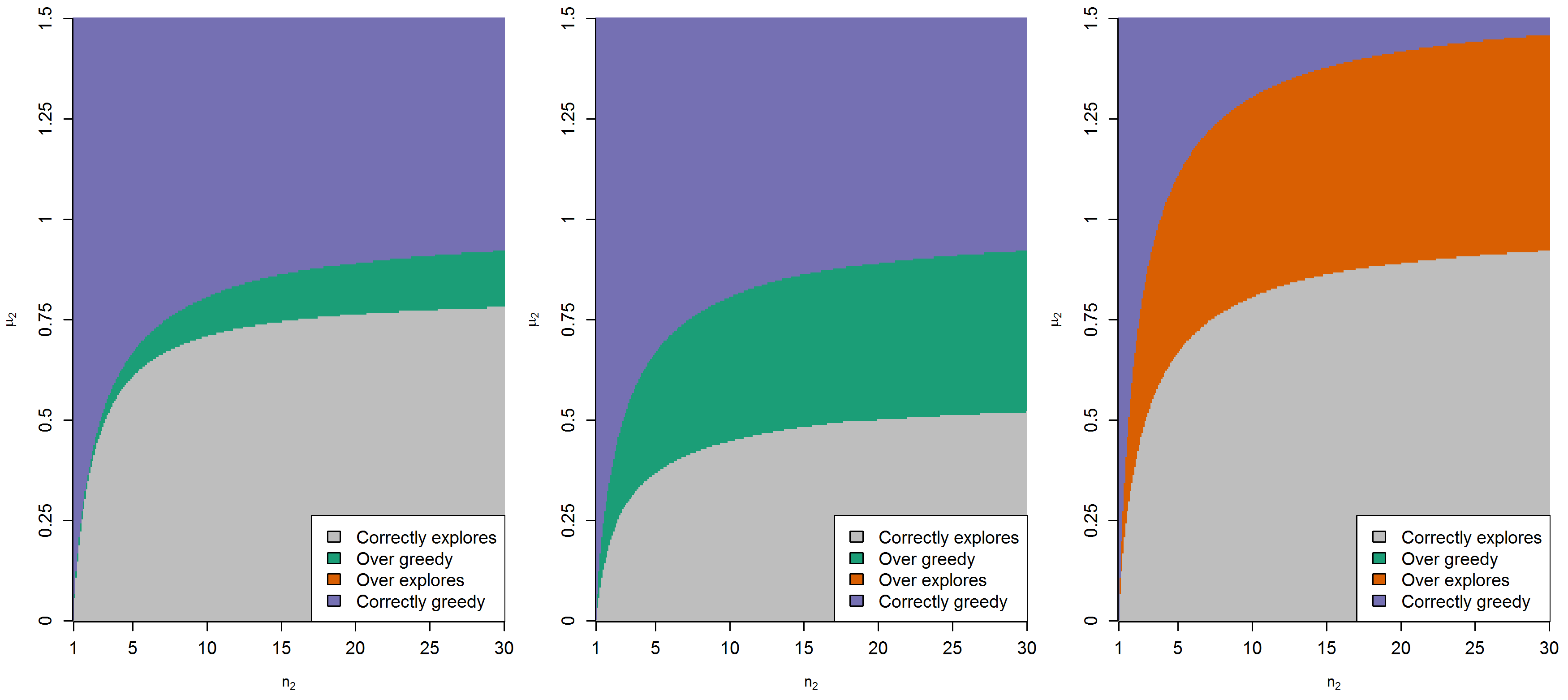}
\caption{Plots of actions for KGI, GIBL and GICG (from left to right) for
$\gamma=0.95$, $n_1=1$, $\mu_1=0$.}
\label{fig:RLB2}
\end{figure}
\citet{weber1992gittins} enunciates a desirable property for
policies which is enjoyed by the optimal GI policy. It can be thought of as a generalised
stay-on-a-winner property.

\begin{definition}
A policy is index consistent if, once an arm is chosen then it continues to
be chosen while its Gittins index remains above its value at the start of the
period of continuation.
\end{definition}

The region of over-exploration in the \textit{RLB }plot in Figure \ref{fig:RLB1}
yields states in which KG is not index consistent. It will simplify the proof and
discussion of the next result if we write the Gittins index for an arm in
state $\left( \Sigma ,n\right) $ as $\nu ^{GI}\left( \Sigma ,n\right) =\frac{%
\Sigma }{n}+l^{GI}\left( n\right) ,$ where $l^{GI}$ is the GI (ie,
optimal) learning bonus for the arm. Note that for notational economy we
have dropped the $\gamma $-dependence from the index notation. It now
follows from the above definition that $R^{GI}\left( n_{1},n_{2}\right)
=l^{GI}\left( n_{1}\right) -l^{GI}\left( n_{2}\right) $. More generally, if $%
\pi $ is an \textit{index policy }we use $l^{\pi }$ for the learning bonus
implied by $\pi ,$ with $R^{\pi }\left( n_{1},n_{2}\right) =l^{\pi }\left(
n_{1}\right) -l^{\pi }\left( n_{2}\right) $.

To prove Proposition \ref{prop:gauss2} we use that each policy over-explores,
as shown in Figures \ref{fig:RLB1} and \ref{fig:RLB2} for KG and GICG and for (e.g.) $\gamma=0.99$
for GIBL (not shown). The idea of the proof is that a policy that over-explores
overestimates the RLB of the arm with lower $n$. After the arm is pulled $n$
increases and its RLB is reduced. There are values of $y$ such that the arm's GI
will increase (as its reduction in RLB is smaller) but its $\mu$ will
not increase sufficiently to overcome the loss of RLB and so the policy
will switch arms. 
\begin{proposition}
\label{prop:gauss2}
Policies KG, GIBL and GICG are not index consistent.
\end{proposition}

\begin{proof}
For definiteness, consider policy KG. From the calculations
underlying Figure \ref{fig:RLB1} we can assert the existence of state $\left( \mathbf{\Sigma ,n}%
\right) $ such that $\Sigma _{1}=0,n_{1}=1,n_{2}=2$ and $2R^{KG}\left(
1,2\right) >\Sigma _{2}>2R^{GI}\left( 1,2\right) $, equivalently, $%
R^{KG}\left( 1,2\right) >\mu _{2}>R^{GI}\left( 1,2\right) ,$ when $\gamma
=0.95$. It follows that $A^{KG}\left( \mathbf{\Sigma ,n}\right) =1$ and KG
over-explores in this state. We suppose that the pull of arm $1$ under KG
in state $\left( \mathbf{\Sigma ,n}\right) $ yields a reward $y$ satisfying $%
\mu _{2}>\frac{y}{2}>R^{GI}\left( 1,2\right) =l^{GI}\left( 1\right)
-l^{GI}\left( 2\right)$. But $\nu ^{GI}\left( y,2\right) =\frac{y}{2}%
+l^{GI}\left( 2\right) >l^{GI}\left( 1\right) =\nu ^{GI}\left( 0,1\right) $
and so the Gittins index of arm $1$ has increased as a result of the reward $%
y$. However, the symmetry of the Normal distribution and the fact that $\mu
_{2}>y$ guarantees that KG will choose arm $2$ in the new state. Thus
while the Gittins index of arm $1$ increases, KG switches to arm $2$ and
hence is not index consistent. Regions of over-exploration for 
GIBL and GICG (in the former case when $\gamma =0.99$) means that a
similar argument can be applied to those policies. This concludes the proof.
\end{proof}

An absence of over-exploration does not guarantee index consistency for a
policy. However, we now give a sufficient condition for an index policy
never to over-explore and to be index consistent.

\begin{proposition} \label{prop:gauss1}
If index policy $\pi $ satisfies $0\leq R^{\pi }\left( n_{1},n_{2}\right)
\leq R^{GI}\left( n_{1},n_{2}\right) $ $\forall n_{1}<n_{2}$ then it never
over-explores and is index consistent.
\end{proposition}
\begin{proof}
Let state $\left( \mathbf{\Sigma ,n}\right) $ be such that $\Sigma _{1}=0$.
This is without loss of generality. For the over-exploration part of the
result, we consider two cases. In case $1$ we suppose that $\mu _{2}>0$ and
the GI policy chooses greedily when it chooses arm $2$. This happens when $%
\mu _{2}\geq R^{GI}\left( n_{1},n_{2}\right) $. If $n_{1}<n_{2}$ then the
condition in the proposition implies that $\mu _{2}\geq R^{\pi }\left(
n_{1},n_{2}\right) $ and policy $\pi $ must also choose arm $2$. If $%
n_{1}\geq n_{2}$ then the condition in the proposition implies that $R^{\pi
}\left( n_{1},n_{2}\right) \leq 0$ and hence that $\mu _{2}\geq R^{\pi
}\left( n_{1},n_{2}\right) $ trivially, which implies that policy $\pi $
continues to choose arm $2$. This concludes consideration of case $1$. In
case $2$ we suppose that $\mu _{2}\leq 0$ and so the GI policy chooses
greedily when it chooses arm $1$. If $n_{1}<n_{2}$ then we have $\mu
_{2}\leq 0\leq R^{\pi }\left( n_{1},n_{2}\right) $ while if $n_{1}\geq n_{2}$
then we must have that $\mu _{2}\leq R^{GI}\left( n_{1},n_{2}\right) \leq
R^{\pi }\left( n_{1},n_{2}\right) \leq 0$. Either way, policy $\pi $ also
chooses arm $1$ and case $2$ is concluded. Hence, under the condition in the
proposition, policy $\pi $ never explores when GI is greedy, and so never
over-explores. For the second part of the result suppose that in state $%
\left( \mathbf{\Sigma ,n}\right) ,$ index policy $\pi $ chooses arm $a$ and
that the resulting reward $y$ is such that $\nu ^{GI}\left( \Sigma
_{a}+y,n_{a}+1\right) >\nu ^{GI}\left( \Sigma _{a},n_{a}\right) ,$ namely
arm $a$'$s$ Gittins index increases. Under the condition in the proposition
we then have that%
\begin{align}
\frac{\Sigma _{a}+y}{n_{a}+1}&+l^{GI}\left( n_{a}+1\right)>\frac{\Sigma
_{a}}{n_{a}}+l^{GI}\left( n_{a}\right)\notag\\
&\Leftrightarrow \frac{\Sigma_{a}+y}{n_{a}+1}-\frac{\Sigma _{a}}{n_{a}}>R^{GI}\left( n_{a},n_{a}+1\right) \geq
R^{\pi }\left( n_{a},n_{a}+1\right)\notag\\
&\Rightarrow \frac{\Sigma _{a}+y}{n_{a}+1}+l^{\pi }\left( n_{a}+1\right) >%
\frac{\Sigma _{a}}{n_{a}}+l^{\pi }\left( n_{a}\right) 
\end{align}%
and we conclude that policy $\pi $ will continue to choose arm $a$. Hence $%
\pi $ is index consistent. This concludes the proof.
\end{proof}

\begin{conjecture}
On the basis of extensive computational investigation we conjecture that
policy KGI satisfies the sufficient condition of Proposition \ref{prop:gauss1}
and hence never over-explores and is index consistent. We have not yet succeeded
in developing a proof.
\end{conjecture}
\subsection{Computational Study}
\label{sec:comp2}
This section gives the results of computational experiments on the MAB with
Normally distributed rewards (NMAB). The same methodology as in Section
\ref{sec:comp1} is used. As well as the basic MAB, also considered are the
finite horizon NMAB with undiscounted rewards (Section \ref{ssec:fhnamb}) and a
problem where arm beliefs are correlated (Section \ref{ssec:cnamb}). It extends
the experiments of \citet{ryzhov2012knowledge} by testing against more competitive policies (including GI) and by separating the
effects of finite horizons and correlated arms. In both of these latter problems the Gittins Index
Theorem (\citet{gittins2011multi}) no longer holds and there is no index
policy that is universally optimal. This raises the question of whether index policies
suffer on these problems in comparison to non-index policies such as KG.

A new parameter, $\tau$ for the observation precision is introduced for several
of these experiments so that $Y_{a}\mid \theta_{a}\backsim N(\theta_{a},\tau)$.
In Section \ref{sec:gauss} it was assumed $\tau=1$ but the results given there
hold for general $\tau$. The posterior for $\theta$ is now updated by
$p(\theta_a|y)=g(\theta|\Sigma_a+\tau y,n_a+\tau)$. We take $\tau$ to be known
and equal for all arms.

\subsubsection{Infinite Horizon, Discounted Rewards}
The first problem compares KG, KGI, GIBL, GICG against the optimal policy
on the standard NMAB over a range of $\tau$. The lost reward as a
percentage of the optimal reward is shown in Figure \ref{fig:nmab}.
\begin{figure}
\centering
\includegraphics[width=\textwidth]{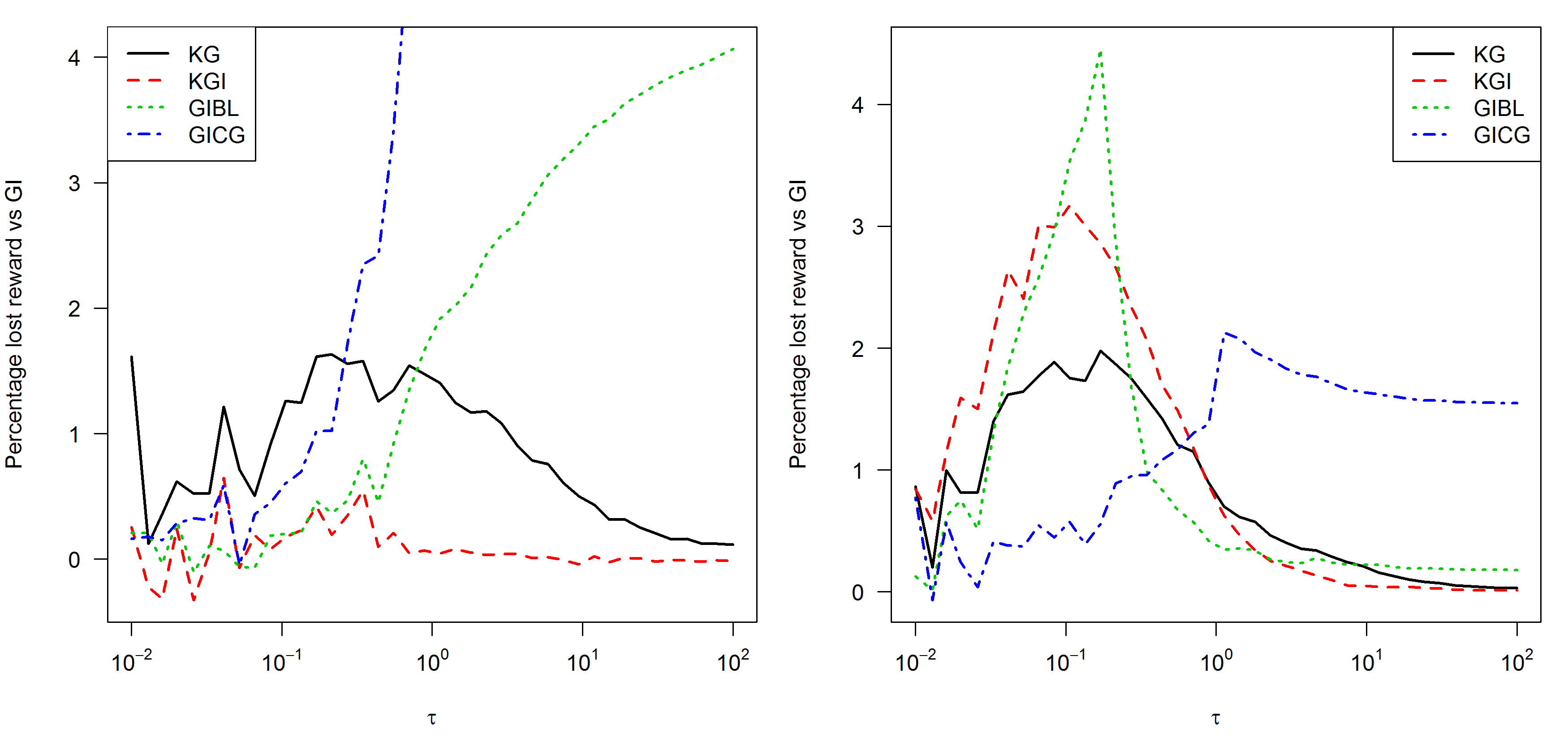}
\caption{Lost reward versus optimal for heuristic policies for the NMAB with
$\gamma=0.9$ (left) and $\gamma=0.99$ (right). There are $k=10$ arms each with a $N(0,1)$
prior.}
\label{fig:nmab}
\end{figure}
The plot does not show the loss for GICG for high $\tau$ and
$\gamma=0.9$ as it is very high relative to the other policies (rising to
$>38\%$). The Greedy policy has similarly poor performance.

\citet{ryzhov2012knowledge} used the GICG policy as a comparator for
the KG policy for the discounted Gaussian MAB. It was described as ``the current state
of the art in Gittins approximation''. These approximations are supposed to be
better than the older GIBL but it appears that the improvements are mainly for
large $n$ which may not result in improved performance. The RLB plots
earlier in this section suggest that the approximations for low $n$ are not
better and it is more important to be accurate in states reached at early times
due to discounting. As $\gamma$ becomes larger these early actions form a smaller
portion of total reward and are therefore less significant.

There is no one best policy for all problem settings. Policy KGI is uniformly
strong for $\gamma=0.9$ but is weaker for $\gamma=0.99$. Both KGI and KG
do well for high $\tau$. This is because more information is gained from a
single observation and so myopic learning policies become closer to optimal.
As $\tau$ becomes smaller it becomes important to consider more future
steps when evaluating the value of information. However when $\tau$ is very low
learning takes so long that a simple greedy approach is again effective.
Hence KG and KGI are weakest for moderate values of $\tau$ between 0.1 and 1,
depending on the number of arms.

\subsubsection{The Finite Horizon NMAB}
\label{ssec:fhnamb}
This section considers a variant on the NMAB where the horizon is a fixed length
and rewards are not discounted (FHNMAB). One strength of KG is that it adapts
easily for different horizon lengths and discount rates. GIBL and
GICG, however, are designed only for infinite horizons. \citet{ryzhov2012knowledge} got
round this problem by treating the discount rate as a tuning parameter. This
allowed them to run experiments on a single horizon length ($T=50$). However, it
is not ideal. Firstly, the tuning parameter will need to be different for
different horizons and there is no simple way to set this. Secondly, the policy
is not appropriate for the problem because a policy for a finite horizon should
be \emph{dynamic}, it should change over time by exploring less as the end time
approaches, whereas this policy is \emph{static} over time. We give a method
here by which any policy designed for an infinite discounted problem can be
adapted to a finite horizon one so that it changes dynamically with time. Note
that all KG variants given in this paper (including KGI) are already dynamic
when the horizon is finite so do not require any adjustment.
\begin{definition}[Finite Horizon Discount Factor Approximation]
A policy which depends on a discount factor $\gamma$ can be adapted to an
undiscounted problem with a finite horizon $T$ by taking
\begin{equation}
\gamma(t,T)=\frac{T-t-1}{T-t},\quad t=0,1,\ldots,T-1.
\end{equation}
\end{definition}
This chooses a $\gamma$ such that $\gamma / (1 - \gamma) = T - 1 - t$ so that
the ratio of available immediate reward to remaining reward is the same in the
infinite case with $\gamma$ discounting (LH side) as the undiscounted finite
case (RH side).

Figure \ref{fig:nmabf} shows percentage lost reward versus KG for KGI and the
adjusted GIBL (with KG shown as a straight line at zero). 

\begin{figure}
\centering
\includegraphics[width=\textwidth]{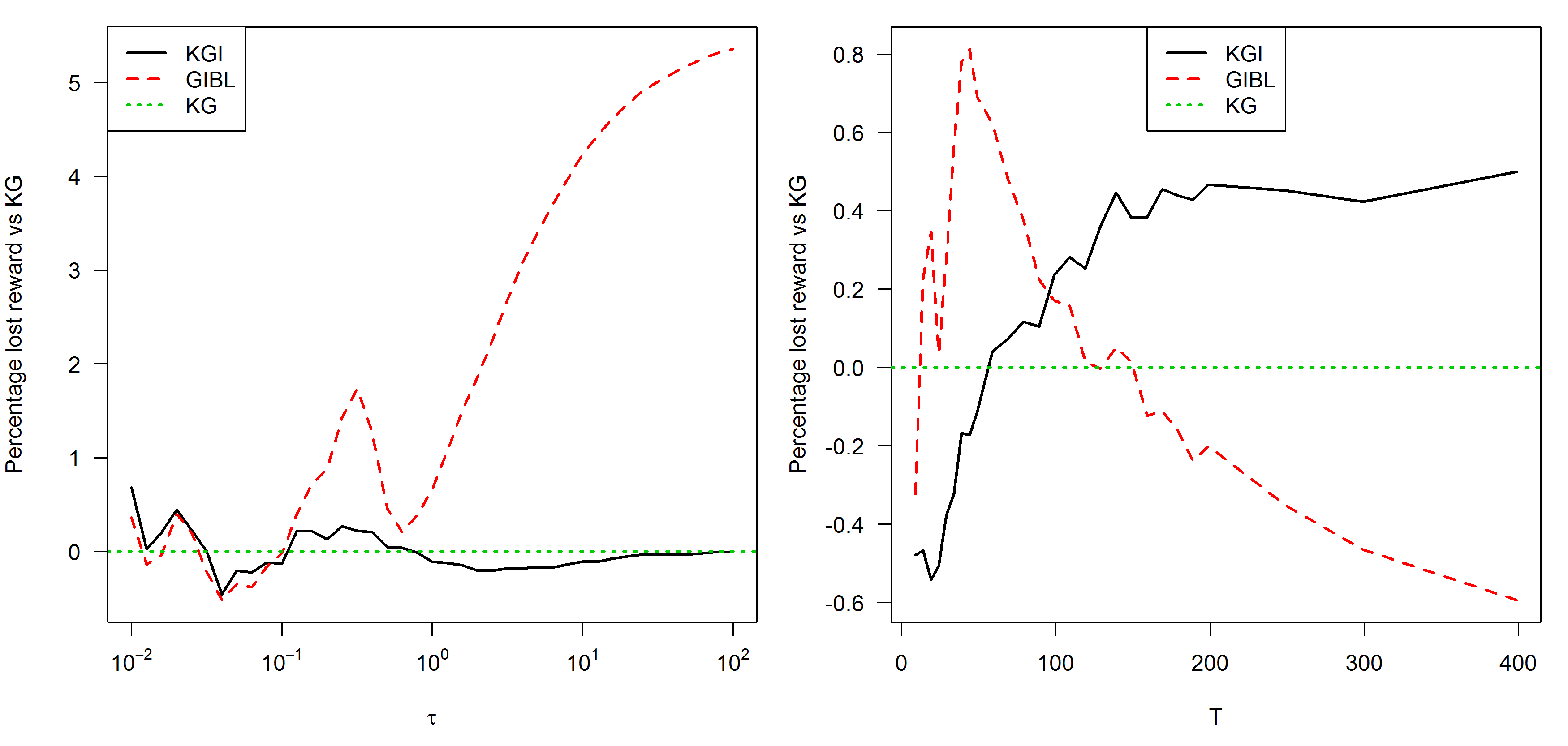}
\caption{Lost reward on the FHNMAB. Left shows performance over a
$\tau\in[10^{-2},10^2]$ with $T=50$, right over $T\leq400$ with $\tau=1$. All
$k=10$ arms had the same $N(0,1)$ prior.}
\label{fig:nmabf}
\end{figure}
Note that the scale of the vertical axis on the right plot is quite close
to zero so that no policies are very distinct from KG here. GIBL shows similar
behaviour to that seen in Figure \ref{fig:nmab} with infinite horizons,
performing similarly to KG at $\tau=1$ (worse for shorter horizons but better
for higher $T$) but doing very badly as $\tau$ increases above 1. KG and KGI show similar
results but KG is the preferred policy for $T\geq 60$.
\subsubsection{The Correlated NMAB}
\label{ssec:cnamb}
The NMAB variant where arm beliefs are correlated was studied in
\citet{ryzhov2012knowledge} where the ability of KG to handle correlated beliefs
was given as a major advantage of the policy. However, being able to incorporate
correlation into the KG policy does not mean performance will improve and the
experimental results that were given were mixed. A further short experimental
study is conducted here for several reasons. Firstly, as shown earlier in this
section, the GI approximation used (GICG) performs poorly in many circumstances
and the GIBL and KGI policies might offer a stronger comparison. Secondly,
\citet{ryzhov2012knowledge} used the finite horizon undiscounted version of the
problem. As described earlier the policies based on GI are not designed for this
problem so an artificial tuning parameter was introduced. Here we use infinite
horizon with discounted rewards as before. This makes it clearer to see the
effect of the introduction of correlation without the extra complication of the
different horizon.

The problem is the same as the NMAB described in Section \ref{sec:gauss}
except that beliefs are captured in a single multivariate Normal distribution
for all the arms rather than one univariate Normal for each arm. For each
simulation run the $\theta$ values for all arms are drawn from this true
multivariate prior. The belief correlation structure can take many different
forms but here we use the same the power-exponential rule used in
\citet{ryzhov2012knowledge}. Prior covariances of the variance-covariance matrix
$\boldsymbol{C}$ are
\begin{equation}
\boldsymbol{C}_{i,j}=e^{-\lambda(i-j)^2}.
\end{equation}
where the constant $\lambda$ determines the level of correlation (decreasing
with $\lambda$). In the experiments here all prior means are zero.

Two versions of KG are tested, the complete version that incorporates the
correlation (CKG) and a version that assumes the arms are independent (IKG).
Details of the CKG policy are given in \citet{ryzhov2012knowledge} using an
algorithm from \citet{frazier2009knowledge}. All policies tested (including
IKG) use the true correlated model when updating beliefs but, apart from CKG,
choose actions that make the false assumption that arms are independent.
Updating beliefs using the independence assumption results in much poorer
performance in all cases and therefore these results are not shown. CKG is
significantly slower than the other policies, scaling badly with $k$. This
limits the size of the experiment so 40000 runs are used.
The results over $\lambda\in[0.05,1]$ are shown in Figure \ref{fig:nmabc}.
\begin{figure}
\centering
\includegraphics[width=\textwidth]{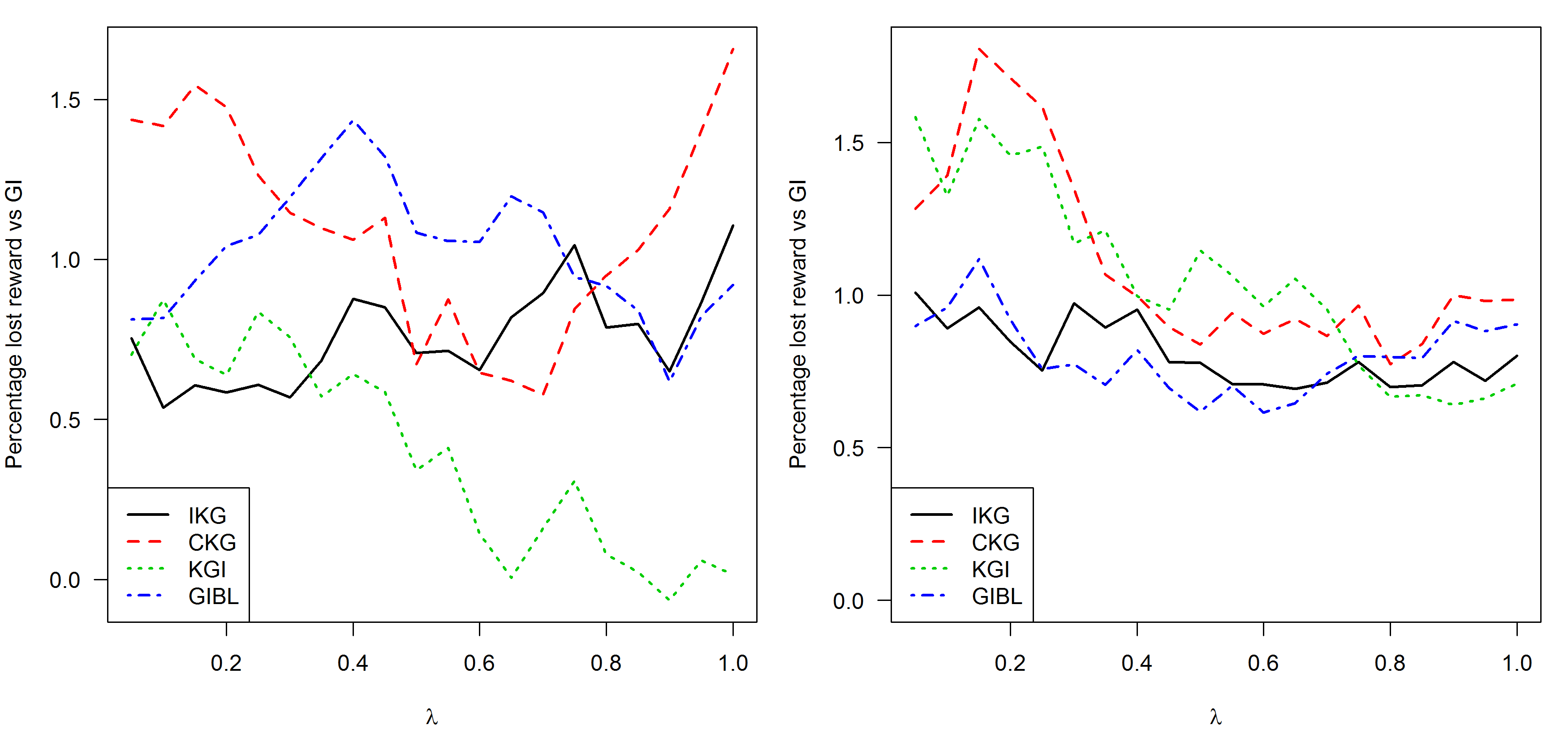}
\caption{Lost reward vs GI on the Correlated NMAB for $\lambda\in[0.05,1]$
with $\gamma=0.9$ (left) and $\gamma=0.99$ (right). Both use $k=10$ with
$\tau=1$.}
\label{fig:nmabc}
\end{figure}
The first observation is that, although GI is not optimal for this problem, it
still clearly outperforms all the other heuristics indicating that using a index
policy is not an obvious handicap. The GI approximation policies'
performance follows a similar pattern to the independent NMAB with KGI
stronger at $\gamma=0.9$ and GIBL stronger at $\gamma=0.99$. IKG compares well
to both these policies but again there is no evidence that non-index methods are
stronger than index methods. More surprising is that CKG is clearly inferior to
IKG. \citet{frazier2009knowledge} found CKG to be stronger in the offline
problem but this does not appear to translate to the online problem. Exactly why
this is so is not clear as this is a difficult problem to analyse. CKG
requires $\mathcal{O}(k^2\log(k))$ to compute compared to IKG
which requires $\mathcal{O}(k)$. CKG's performance would have to be much better
to justify its use and in many online problems with larger $k$ it would simply not
be practical while IKG and the three simple index policies all scale well.

These experiments only scratch the surface of the correlated MAB
problem and there are a number of possible issues. As rewards are observed the
correlations in beliefs reduce and the arms tend to independence. Therefore
correlations will be most important in problems with short horizons or steep discounting. Secondly, the number of arms used here is quite small. This is
partly because CKG becomes very slow as $k$ increases so its use on larger
problems would not be practical. A feature of correlated arm beliefs is that we
can learn about a large number of arms with a single pull and therefore
independence assumptions should be punished with greater numbers of arms.
However we still learn about multiple arms as long as belief updating is
handled accurately which is easy to do in this Gaussian setting. If this is not
done then learning will be much slower and we did find that it is
important that belief updating incorporates correlations.

One difficulty with analysing policies on this problem is that it still
matters that the policy is effective on the basic MAB problem. Suboptimality in
this regard can mask the effect of introducing correlation and changes may
improve or worsen performance quite separately from addressing the issue of
correlations. For example if a policy normally over-explores then any change
that makes it greedier might improve performance. Thompson Sampling is a policy
that can easily incorporate correlations (by sampling from the joint posterior)
but the high level of exploration that comes from randomised actions does not do
well on short horizon problems and any changes due to correlations will be too subtle to
change that.

\section{Conclusion}
\label{sec:conclusion}
We identify an important class of errors, \emph{dominated actions} which are
made by KG. This involves choosing arms that have \emph{both} inferior
exploitative and explorative value. Much of the existing work on KG has focused
on Gaussian rewards but these have features (symmetric and unbounded
distributions) that avoid the domination problem. For other reward
distributions the performance of KG can suffer greatly. Two new variants are
given which remove this problem. Of these, NKG is simpler while PKG gives better
experimental results by correcting errors besides those arising from dominated actions.

We also introduced an index variant of KG which avoids dominated actions, which
we called KGI. For problems where the optimal policy is an index policy,
simulation studies indicate that KGI is more robust than other proposed index
policies that use approximations to the optimal index. It has computational
advantages over KG and performed competitively in empirical studies on all
problems tested including those where index methods are known to be suboptimal.
One such problem is the MAB with correlated beliefs. Although KG can incorporate
these beliefs it was found that any performance gain was, at best, small and did
not justify the extra computation involved.

The new variants we introduce give a range of simple heuristic policies, of both
index and non-index type. On the problems tested here at least, there did not
appear to be any advantage to using non-index methods and, in addition, index
methods have computational advantages on some BSDPs. However, this may not
always be the case and further research will be needed to be more confident on
other variants of this problem.

\section*{Acknowledgements}
We gratefully acknowledge the support of the EPSRC funded EP/H023151/1 STOR-i
centre for doctoral training. We also thank the reviewers for their careful
reading of the paper and their helpful and thoughtful comments.
\bibliographystyle{myapalike}
\bibliography{KG_paper}{}
\end{document}